\documentclass{article}

\usepackage[preprint]{neurips_2026}

\usepackage[utf8]{inputenc} 
\usepackage[T1]{fontenc}    
\usepackage{hyperref}       
\usepackage{graphicx}
\usepackage{url}            
\usepackage{booktabs}       
\usepackage{dsfont}
\usepackage{amsfonts}       
\usepackage{nicefrac}       
\usepackage{microtype}      
\usepackage{xcolor}         
\usepackage{natbib}
\usepackage{amsthm}
\usepackage{pifont}
\usepackage{pifont}
\usepackage{multirow}
\usepackage{float}
\usepackage{enumitem}
 \usepackage{textcomp}
\usepackage{algorithm}
\usepackage{authblk}
\usepackage{algorithmic}
 \usepackage{amsmath}

\newtheorem{theorem}{Theorem}

\title{SMILE: Bridging Continuous Optimization and Discrete Symbolic Recovery}

\author[1]{\textbf{Mansooreh Montazerin}}
\author[1]{\textbf{Antonio Ortega}}
\author[2,$\dagger$]{\textbf{Ajitesh Srivastava}}
\affil[1]{Department of Electrical and Computer Engineering, University of Southern California}
\affil[2]{Bouvé College of Health, Khoury College of Computer Science, Network Science Institute, Northeastern University}
\affil[ ]{\texttt{\{mmontaze, aortega\}@usc.edu, aji.srivastava@northeastern.edu}}

\begin{document}

\maketitle

\footnotetext{Part of this work was done while affiliated with the University of Southern California.}

\begin{abstract}
  
Symbolic regression (SR) discovers closed-form mathematical expressions from data, offering interpretability beyond black-box models. Existing methods suffer from slow convergence in combinatorial search spaces and lack mechanisms to exploit compositional structure in the data. We introduce SMILE (\textbf{S}ine, \textbf{M}ultiplication, \textbf{I}dentity, \textbf{L}ogarithm, \textbf{E}xponential), a hybrid framework that unifies continuous gradient-based optimization with discrete symbolic recovery through three stages: structural analysis of the data to identify the compositional hierarchy of the target expression, continuous optimization to learn parameters of a network that encodes the target expression using interpretable activations, and symbolic recovery through structured pruning, coefficient optimization, and rounding. This final stage distills the learned network into a compact expression with exact symbolic constants. We evaluate SMILE on SRBench across ground-truth and black-box datasets, with ablation studies validating each component. SMILE achieves the highest symbolic solution rate at the largest noise levels, demonstrating strong robustness where competing methods degrade substantially. It consistently lies on the Pareto front of accuracy versus complexity, recovering significantly simpler expressions in a fraction of the time required by the competing methods.
\end{abstract}

\section{Introduction}\label{sec:intro}

Uncovering the mathematical laws governing complex systems is a fundamental goal in science and engineering. Symbolic regression (SR) addresses this challenge by searching over both the structure and parameters of mathematical expressions to discover interpretable closed-form relationships from observed data~\cite{physics2023,marin2023}. Unlike conventional regression, which fits parameters within predefined models, and neural networks, which yield opaque representations despite strong predictive performance~\cite{cranmer-interp}, SR searches for both the structure and parameters of equations, offering a flexible, interpretable, and model-agnostic alternative~\cite{marin2023}. It has become an essential tool across scientific disciplines including fluid mechanics, molecular systems, astrophysics, and materials science~\cite{sr-astro,sr-pde,sr-dynamic,sr-material}. However, the vast combinatorial search space of candidate equations makes discovering meaningful representations inherently challenging, as overly complex models risk overfitting while overly simplistic ones may fail to capture essential dynamics~\cite{lacava}.


Existing approaches to SR can be broadly categorized based on the search space they navigate. \textit{Discrete methods}, such as genetic programming (GP) \cite{sr-gp2022, sr-gp2023}, reinforcement learning \cite{rl-spars2023,rl-spars2025}, and neural sequence-to-sequence models \cite{biggio-seq2seq2020,srscience2026}, explore the space of symbolic expressions by iteratively refining candidate solutions constructed from primitive operations. However, relying on predefined operations and sequential generation often yields overly complex models, and the symbolic discontinuity of these combinatorial landscapes prevents gradient-based guidance, leading to slow convergence. \textit{Continuous methods} reformulate SR as a numerical optimization problem, enabling gradient-based training over smooth parameter spaces. Recent approaches range from sparse regression over predefined libraries~\cite{sindy} and neural architectures with symbolic activations~\cite{eql2018,parfam} to deep learning frameworks that treat expressions as sequences, structured graphs, or samples from learned latent spaces~\cite{sr-trans2023,cranmer-interp,sr-gen2023}. While offering better scalability, these methods often require large amounts of training data, struggle with generating syntactically valid or semantically meaningful expressions, require user-specified structural choices, and offer limited interpretability during training.
Beyond these two categories, several methods incorporate elements of both paradigms — for instance, AI Feynman \cite{aifeynman} pairs continuous neural network fitting with discrete physics-inspired decomposition, while uDSR \cite{udsr} unifies neural-guided search, sparse regression, and evolutionary refinement. In these approaches, however, the discrete and continuous steps remain structurally decoupled, and the data is treated as a fitting target rather than as a source of structural cues, such as separability or compositional hierarchy, that can guide the search toward expressions consistent with physical principles.

To address these limitations, we propose \textit{SMILE}, a \textit{hybrid framework} that formulates SR as the structural optimization of a neural architecture, \textit{unifying continuous training with discrete symbolic recovery}. SMILE begins with a structural analysis of the input data through numerical approximations to mathematical operations, such as variable marginalization and residual analysis, to identify the compositional hierarchy of the underlying formula, progressively narrowing the space of candidate expressions and determining how the problem is decomposed. At its core, SMILE is a feedforward network whose hidden layers each contain five neurons with fixed symbolic activations: \textit{Sine (S)}, \textit{Multiplication (M)}, \textit{Identity (I)}, \textit{Logarithm (L)}, and \textit{Exponential (E)}, primitives reflecting the fundamental operators most frequently encountered in physical and scientific laws. Because the chosen primitives jointly span a broad class of mathematical operations, SMILE can represent complex expressions with significantly fewer layers than standard neural networks, ensuring high efficiency and transparency in the search process. We demonstrate that SMILE is a universal approximator~\cite{universal,universal1} as the combination of identity, exponential, and logarithm neurons enables SMILE to reproduce any MLP with sigmoid activations and hence approximate any continuous mapping. Symbolic recovery then proceeds through structured pruning which progressively removes connections with the least impact on accuracy, enabling the recovery of interpretable expressions. The remaining coefficients are then refit via parametric optimization and rounded to produce expressions consistent with the coefficients typically found in physical and mathematical laws. By combining gradient-based optimization with an inherently symbolic architecture and a structured recovery pipeline, SMILE achieves fast convergence without combinatorial search, requires no user-specified structural hyperparameters, and produces clean closed-form expressions directly from the trained network. 

Specifically, our main contributions are as follows: \textit{(i)} We introduce SMILE, an SR framework with five fixed interpretable activations that unifies continuous optimization with discrete symbolic recovery through a three-stage pipeline: \textit{structural analysis, continuous optimization, and symbolic recovery}. To the best of our knowledge, this data-driven structural analysis has not been explored in prior work. \textit{(ii)} We prove the expressive power of SMILE as a universal approximator and develop a tailored loss function to promote sparsity and prevent gradient instability, combined with a gating mechanism that enables structured pruning for interpretable symbolic recovery. \textit{(iii)} We conduct extensive experiments on standard SR benchmarks, demonstrating that SMILE achieves competitive exact symbolic recovery with strong robustness to noise, attaining the highest symbolic solution rate at the largest noise level, significantly lower expression complexity, and discovery time reduced from hours to minutes. Ablation studies further validate the contribution of each pipeline component.

\begin{figure}[tb]
\centering
\includegraphics[trim=0 0 0 0, clip, width=1\textwidth]{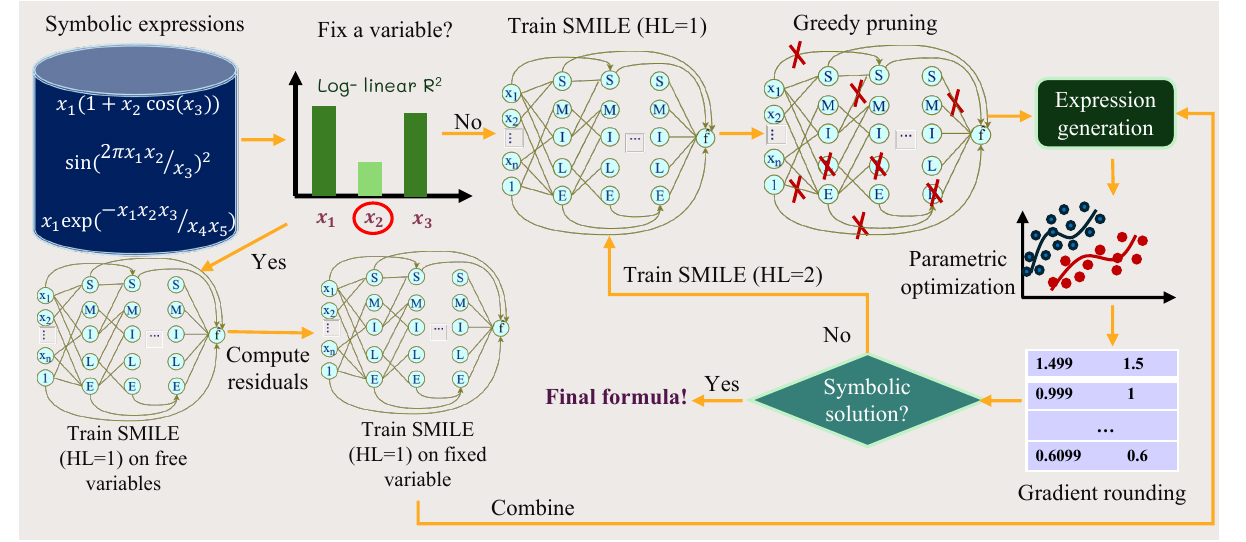}
\caption{End-to-end pipeline of the SMILE framework. The pipeline performs a log-linear $R^2$ scan to classify the problem into one of three cases: \textbf{(C1)} direct training, \textbf{(C2)} training on the inverted target, or \textbf{(C3)} decomposition into simpler subproblems. In all cases, the trained network is passed through a recovery pipeline consisting of greedy pruning, symbolic extraction, parametric optimization, and gradient-based coefficient rounding. HL denotes the number of hidden layers.}
\label{fig:pipeline}
\end{figure}

\section{Related work}\label{sec:related}

SR is a foundational technique to discover closed-form equations that best explain observed input–output relationships in the data. SR has been studied extensively through a variety of approaches that differ primarily in how they explore the space of candidate expressions.

\textbf{Discrete evolutionary search}
Genetic programming remains the most established paradigm for SR, dating back to the foundational work of Koza~\cite{koza-genetic1994}. GP evolves populations of tree-structured expressions through biologically inspired operations such as mutation, crossover, and selection. PySR~\cite{sr-gp2023} is a more recent, high-performance GP-based system that introduces a multi-population evolutionary algorithm with an evolve–simplify–optimize loop. Beyond GP, Monte Carlo Tree Search~\cite{sr-mcts2022} explores expression trees through reward-guided rollouts, while combinatorial optimization techniques~\cite{rag-sr} formulate expression discovery as a constrained search problem. GP-based methods, however, rely on heuristic exploration of a large combinatorial space, often leading to high computational cost, sensitivity to hyperparameters, and overly complex expressions.

\textbf{Neural generation and reinforcement learning} 
A parallel line of work leverages machine learning to guide symbolic search while still operating in a discrete space. RL approaches~\cite{rl-spars2023,rl-spars2025,egg-sr} train agents to sequentially construct expressions by selecting operators and operands from a predefined library. While the agent's policy is optimized via continuous objectives, the search itself remains discrete, as each step involves a categorical choice over symbolic tokens, which can lead to slow convergence and sensitivity to reward design. Neural encoder-decoder models~\cite{biggio-seq2seq2020} and transformer-based architectures such as E2E~\cite{e2e} approach SR as a direct mapping from numerical observations to symbolic expressions, leveraging pretraining on synthetic equation-data pairs. However, these models are particularly sensitive to their training distribution and often fail on expressions outside the patterns seen during pretraining~\cite{sr-trans2024}. More recently, large language models (LLMs)~\cite{srscience2026,lasr,llm-sr} have been applied to SR by leveraging internalized mathematical knowledge to generate candidate expressions through prompting. However, recent studies~\cite{llm-srbench,lasr} have shown that LLM-based methods often recall known equations from their training corpus rather than performing genuine discovery, making a fair comparison with methods that learn exclusively from the provided data infeasible. LLM-based approaches, therefore, fall outside the scope of our comparisons.

\textbf{Continuous and differentiable search}
An alternative paradigm reformulates SR as a continuous optimization problem. Sparse regression approaches such as SINDy~\cite{sindy} select a minimal subset from a predefined library via sparsity-promoting optimization, offering efficiency but limiting expressivity to the chosen dictionary. AI Feynman~\cite{aifeynman} combines neural network fitting with physics-inspired decomposition strategies, recursively simplifying problems before applying symbolic search. uDSR~\cite{udsr} integrates neural-guided search, sparse regression, and evolutionary refinement, achieving strong performance but at high computational cost and with limited interpretability, as the final expression is assembled from multiple independent strategies.
Among continuous approaches, EQL$^{\div}$~\cite{eql2018,marin2023} employs a neural architecture with fixed symbolic activations but imposes structural constraints such as restricting division to the final layer, requires multiple layers for common functions, and provides no framework for extracting compact formulae~\cite{sr-dl2019}. ParFam~\cite{parfam} parameterizes candidate expressions as compositions of rational functions optimized via basin-hopping, but both the activation functions and polynomial degrees are user-specified, and it relies on generic $\ell_1$ regularization without a mechanism for resolving floating-point coefficients into exact symbolic constants.

In contrast, SMILE uses a fixed set of symbolic activations across all problems, eliminating the need for operator selection and structural hyperparameter tuning. It distinguishes itself from methods that proceed directly to fitting by first analyzing the compositional structure of the data, training a compact network to fit the target, and then recovering closed-form expressions through structured pruning and rounding, yielding high symbolic recovery rates across standard benchmarks.

\section{Methodology}\label{sec:method}

The SMILE framework approaches SR by combining a compact neural architecture with interpretable activations with a data-driven pipeline comprising structural analysis, continuous optimization, and symbolic recovery. Rather than searching over an exponentially large space of candidate expressions, SMILE first \textit{analyzes the data to identify the compositional structure of the target expression} and selects an appropriate training configuration. This reduces the effective degrees of freedom of the network and simplifies the recovery problem. 
The network is then trained via continuous optimization, after which a sequence of discrete steps — greedy pruning, parametric optimization, and gradient-based rounding — progressively distills the trained network into an exact symbolic expression. We describe each component below, beginning with the problem formulation (\autoref{subsec:pre}), followed by the network architecture (\autoref{subsec:smile}), the training procedure (\autoref{subsec:training}), and the full pipeline (\autoref{subsec:pipeline}).

\subsection{Preliminaries}\label{subsec:pre}

In data-driven equation discovery, our objective is to find a compact and interpretable mathematical expression $\hat{f}$ that closely approximates an unknown target function \( f: \mathbb{R}^d \to \mathbb{R},\) which maps a 
$d$-dimensional input vector $\mathbf{x}$ to an output $y$. 
Given a dataset of $n$ samples \(D = \{(\mathbf{x}_i, y_i)\}_{i=1}^{n}\), SR seeks the underlying mathematical relationship between the input-output pairs such that $\hat{f}(\mathbf{x}_i)\approx y_i$ for all data samples. 
In contrast to standard regression, SR aims not only to fit the observed data accurately but also to discover an interpretable expression that generalizes to unseen inputs. We consider an expression interpretable when it is composed of primitive mathematical operations with few terms and exact symbolic constants. The SMILE framework achieves this by recovering expressions that lie on the Pareto front of accuracy and complexity, favoring the simplest form that accurately represents the underlying solution.

\subsection{The proposed SMILE architecture}\label{subsec:smile}

Unlike prior approaches that generate mathematical expressions through population-based search or sequential modeling, we design a neural network with fixed symbolic activations whose depth and training configurations are selected based on structural properties of the data (\autoref{subsec:pipeline}), and whose weights are jointly optimized to naturally simplify into symbolic expressions. The SMILE architecture is a multi-layer feedforward network in which each hidden layer contains five neurons, one for each of the primitive activations introduced in \autoref{sec:intro}. Four of these --- sine, identity, logarithm, and exponential --- are standard unary functions applied element-wise to a linear combination of their inputs. The fifth, the multiplication neuron, produces a learned product of its inputs with trainable exponents:
\begin{equation}
    m^{(i)} = \exp\!\left(\mathbf{w}^{(i)} \cdot \log\!\left(\mathbf{z}^{(i-1)}\right)\right) = \prod_j \left(z_j^{(i-1)}\right)^{w_j^{(i)}},
\end{equation}
where $m^{(i)}$ is the output of the multiplication neuron, $\mathbf{w}^{(i)}$ are the learnable real-valued weights at layer $i$, $\mathbf{z}^{(i-1)}$ is the output vector of the previous layer, and $z_j^{(i-1)}$ denotes its $j$-th component. This formulation enables the network to directly represent operations such as $x^2$, $\sqrt{x}$, $x_1 x_2$, or $x_1 / x_2$ within a single neuron. Crucially, expressing multiplication as the composition of logarithms and exponentials keeps the entire computation within the same differentiable primitives that define the rest of the network, allowing all weights, including the exponents, to be trained end-to-end via standard gradient-based optimization. Additionally, we include fully dense residual connections between all layers of the network to improve gradient flow and help the network learn a broader range of mathematical expressions in its outputs. We do not include bias terms in the layers; instead, we append a learnable bias of $1$ to the input layer.


The expressive power of the SMILE architecture is justified by the following theorem.
\begin{theorem}
SMILE networks are universal approximators.
\end{theorem}

\begin{proof}[Proof sketch]
We show that the SMILE network can emulate any multilayer perceptron (MLP) with sigmoid activations. A single SMILE block can compute the sigmoid function $\sigma(a) = 1/(1 + \exp(-a))$ using only the identity, exponential, and logarithm neurons, as illustrated in \autoref{fig:sig}a. The identity neurons can also reproduce arbitrary affine transformations between layers. By stacking $L$ such blocks (\autoref{fig:sig}b), we recover an $L$-layer MLP with sigmoid activations, which is a known universal approximator~\cite{universal,universal1}. The full proof is given in~\autoref{app:universality}.
\end{proof}

\begin{figure}[h!]
\centering
  \includegraphics[trim=0 0 0 0, clip, width=0.95\linewidth]{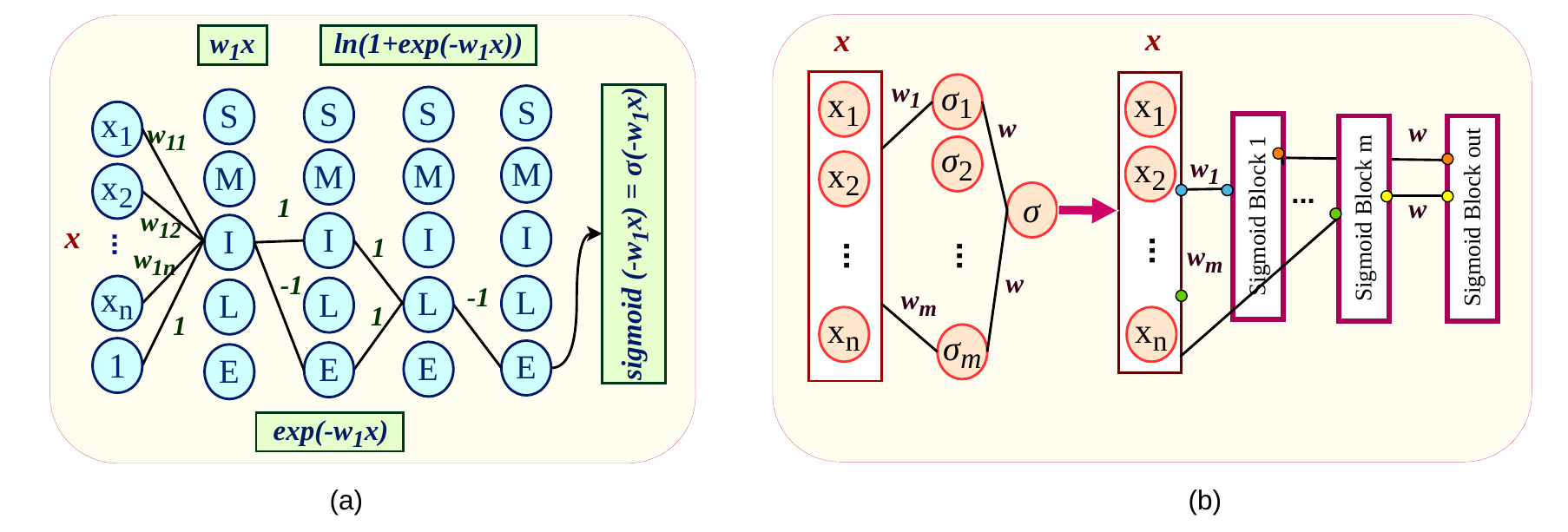}
  \caption{(a) Approximation of the Sigmoid function using a SMILE-based configuration. (b) Emulation of an $L$-layer MLP into an equivalent SMILE network using stacked SMILE layers and Sigmoid blocks.}
  \label{fig:sig}
\end{figure}

While the SMILE architecture is expressive enough to represent a wide range of symbolic expressions, not all neurons and connections are necessary for any given target function. To retain only the components that contribute to the target expression, we include a differentiable gating mechanism within the network that enables the pruning stage to systematically remove inactive paths:

\textbf{Gating mechanism}
To encourage the network to use only the connections and activations necessary to represent the target expression, we introduce two types of learnable gates in each layer, inspired by differentiable gating mechanisms for sparse structure learning~\cite{gate2018}. An \textit{edge gate} $G^{(i)}$, sharing the same shape as the weight matrix $\mathbf{W}^{(i)}$, modulates individual connections by replacing the raw weight with $\mathbf{W}^{(i)} \odot \sigma(G^{(i)})$, where $\sigma$ denotes the sigmoid function. An \textit{activation gate} $A_k^{(i)}$, $k \in \{1, \dots, 5\}$, one per neuron, scales the entire neuron output by $\sigma(A_k^{(i)})$. Both gate logits are initialized at zero, corresponding to $\sigma = 0.5$, and are learned jointly with the network weights. Conceptually, edge gates enable pruning of individual connections, while activation gates enable a stronger form of pruning by removing entire neurons from the computation graph. The sigmoid provides a smooth, differentiable relaxation of discrete selection, allowing gate values to be optimized jointly with the network weights via standard gradient-based training. When combined with sparsity-inducing penalties in the loss function (\autoref{subsec:training}), this design encourages gates to converge toward zero or one, yielding naturally sparse networks for symbolic extraction.

\subsection{Training and optimization}\label{subsec:training}
With the architecture and gating mechanism defined, we train the SMILE network end-to-end using gradient-based optimization. The goal of training is to learn the weight matrices and gate parameters of the SMILE network such that the composition of symbolic activations, governed by these learned coefficients, captures the underlying mathematical relationship in the data. We define the objective function as the sum of three terms, detailed below. 

The principal part of the objective function is the \textit{mean squared error} to fit the observed data by minimizing the $\mathcal{L}_{\text{MSE}} = \frac{1}{N} \sum_{i=1}^{N} ( \hat{y}_i - y_i )^2$, where $\hat{y}_i$ is the predicted output, $y_i$ is the target value, and $N$ is the batch size.

The logarithmic and exponential activations pose numerical challenges during training: the exponential function and its derivatives grow rapidly at high input values, leading to gradient instability, while the logarithm is undefined for non-positive inputs and its derivatives become unstable near zero. To address these issues while preserving differentiability, we apply input clamping to both functions:
$\operatorname{L}(x) = \ln(\max(x, x_l))$ and $\operatorname{E}(x) = \exp(\min(x, x_e))$,
where $x_l$ and $x_e$ are predefined thresholds. To discourage the network from relying on clamped regions, we \textit{penalize inputs that fall outside the valid domain}:
\begin{equation}
\mathcal{L}_{\text{clamp}} = \sum_{i} \mathds{1}_{\{x_{i,\log} < x_l\}} \cdot \left| x_l - x_{i,\log} \right| + \sum_{i} \mathds{1}_{\{x_{i,\exp} > x_e\}} \cdot \left| x_e - x_{i,\exp} \right|,
\end{equation}
where $x_{i,\log}$ and $x_{i,\exp}$ denote the inputs of sample $i$ to the logarithmic and exponential activations, respectively.

As described in~\autoref{subsec:smile}, the SMILE network uses edge and activation gates to control weight and node sparsity. To drive these gates toward binary decisions during training, we denote $\mathcal{G}$ as the set of all gate logits across every layer and define the \textit{gate penalty} as:
\begin{equation}
\mathcal{L}_{\text{gate}} = \sum_{g \in \mathcal{G}} \left|\sigma(g)\right| + \sum_{g \in \mathcal{G}} \sigma(g)\left(1-\sigma(g)\right),
\end{equation}
where $\sigma(\cdot)$ is the sigmoid function. The first term applies $\ell_1$ shrinkage, pushing gate values toward zero. The second term penalizes intermediate gate values — since $\sigma(g)(1-\sigma(g))$ is maximized at $\sigma(g) = \tfrac{1}{2}$, it encourages gates to converge to a clean binary decision of either fully open or fully closed.

Finally, the total loss is $\mathcal{L}_{\text{total}} = \mathcal{L}_{\text{MSE}} + \lambda_{\text{clamp}} \, \mathcal{L}_{\text{clamp}} + \lambda_{\text{gate}} \, \mathcal{L}_{\text{gate}}$, where $\lambda_{\text{clamp}}$ and $\lambda_{\text{gate}}$ control the relative contribution of each penalty term.

\subsection{Data-driven discovery pipeline}\label{subsec:pipeline}

As outlined at the beginning of this section, the structural analysis stage infers properties of the unknown expression through mathematical reasoning and determines whether to decompose the problem into simpler subproblems. Following continuous optimization, the symbolic recovery stage distills the learned network into a compact closed-form expression through pruning, parametric optimization, and coefficient rounding. Each step is described below, and the complete procedure is formalized in Algorithms~\ref{alg:stage1} and~\ref{alg:stage2} in~\autoref{app:pipeline}.


\textbf{Variable analysis and decomposition}
Many target functions in scientific domains exhibit compositional structure, where individually simple components combine through multiplication or addition to form expressions that are difficult to recover in a single pass~\cite{dc-sr}. A broad class of such functions can be written as $f(\mathbf{x}_{\setminus k}) \cdot g(x_k)$ or $f(\mathbf{x}_{\setminus k}) + g(x_k)$, where $f$ and $g$ are individually simple but their combination is not. Rather than increasing network depth to capture such expressions directly, we decompose the problem along the variable responsible for the additional structure. The pipeline handles three cases: \textbf{(C1)} the expression is already a simple product of powered variables and can be recovered directly, \textbf{(C2)} the expression contains a denominator that causes numerical instability during training, in which case we invert the target to convert it into a more tractable form, and \textbf{(C3)} a variable enters the expression through an additive or multiplicative composition, in which case the problem is split into two simpler subproblems. To determine which case applies, we analyze the dataset $D = \{(\mathbf{x}_i, y_i)\}_{i=1}^{n}$ by independently regressing $\log|y|$ against $\log|x_k|$ for each input variable $x_k$. This tests whether the target function follows a power-law relationship across all variables. If all variables yield comparable high $R^2$ values, the expression falls under case \textbf{(C1)} and the pipeline proceeds directly to training. For expressions involving denominators, the clamping on the logarithm and exponential activations can cause relevant weights to shrink prematurely before the network learns a meaningful structure. Therefore, if the log-linear fit is strong, the pipeline additionally tests whether inverting the target improves the conditioning of the problem by comparing low-degree polynomial fits to $y$ and $1/y$. If the inverted form yields a better fit, the expression falls under case \textbf{(C2)}; training proceeds on $1/y$ which converts denominators into numerators for more stable optimization. The final expression is then inverted back after recovery. 

If one variable's log-linear fit yields a significantly lower $R^2$ than the rest, it indicates that this variable breaks the power-law pattern, triggering case \textbf{(C3)}. The pipeline then fixes this variable at two distinct values to isolate the simpler relationship among the remaining variables. In continuously sampled datasets such as the Feynman benchmark~\cite{aifeynman}, exact matches at a specific value of $x_k$ are rare. We therefore define a window $[x^* - \delta, \; x^* + \delta]$ around the fixation point $x^*$ within which $x_k$ is treated as approximately constant. The key requirement is that the variation in $y$ caused by $x_k$ within this window remains negligible: $|y(\mathbf{x}_{\setminus k}, x^* + \delta) - y(\mathbf{x}_{\setminus k}, x^* - \delta)| / \mathrm{std}(y) < \tau$ for fixed $\mathbf{x}_{\setminus k}$. By first-order approximation, this reduces to $\left|\frac{\partial y}{\partial x_k}\right| \cdot 2\delta \, / \, \mathrm{std}(y) < \tau$ (details in Appendix~\ref{app:varfix}). The largest window satisfying this criterion is selected automatically, and the data within it is passed to the first SMILE network, which recovers $f(\mathbf{x}_{\setminus k})$ as if $x_k$ were truly fixed.

Once $f(\mathbf{x}_{\setminus k})$ is recovered, we determine whether $x_k$ enters the full expression multiplicatively or additively. At each of the two fixation points, we compute the multiplicative residuals $r = y / f(\mathbf{x}_{\setminus k})$. If the ratio of these residuals across the two fixation points is approximately constant, this indicates that $y \approx f(\mathbf{x}_{\setminus k}) \cdot g(x_k)$, and the relationship is multiplicative. Otherwise, we default to the additive form $y \approx f(\mathbf{x}_{\setminus k}) + g(x_k)$, computing residuals as $y - f(\mathbf{x}_{\setminus k})$. Multiplicative and additive separability tests have been explored in prior work~\cite{aifeynman}, though our approach operates directly on the recovered symbolic subexpression rather than on neural network approximations of the data. A second SMILE network is then trained on the resulting residuals to recover $g(x_k)$. Finally, the two components are combined into a single expression, either $f \cdot g$ or $f + g$, and passed to the recovery pipeline for symbolic extraction.

\textbf{Greedy pruning}
After training, the network encodes a continuous approximation of the target function, but many of the gated connections and neurons may carry negligible contributions. To distill the network into a compact symbolic expression, we apply a greedy pruning strategy that operates on the gates introduced in \autoref{subsec:smile}. At each iteration, we tentatively close the gate with the smallest impact on the training $R^2$ by forcing $\sigma(g) \approx 0$. If the resulting $R^2$ drop remains below a predefined tolerance, the closure is accepted, and the procedure continues. This process iterates until no further gates can be closed without exceeding the tolerance. Since pruning operates on both edge and activation gates, it removes individual connections as well as entire neurons, progressively reducing the trained network into a minimal subgraph. The remaining active weights and neurons are then composed layer by layer to produce a closed-form expression whose coefficients are the floating-point values learned during training.

\textbf{Parametric optimization}
Since pruning alters the network structure, the remaining coefficients may no longer be optimal for the simplified expression. We therefore refit all numerical constants in the extracted expression via least-squares optimization on the training data.

\textbf{Gradient rounding}
The parametric optimization yields an expression with floating-point coefficients, whereas most scientific laws involve simple integers, rationals, or well-known constants such as $\pi$. To recover exact symbolic constants, we introduce a gradient-based rounding mechanism. For each coefficient $c_i$ and a candidate rounded value $r_i$, we estimate the change in the expression's output using a first-order approximation (\autoref{thm:grad-round}, proof in Appendix~\ref{app:rounding}). Rounding is accepted when $|c_i - r_i| \cdot \left| \frac{\partial \hat{f}}{\partial c_i}(r_i, \mathbf{x}) \right| < \tau$ for all $\mathbf{x} \in S$, where $\tau$ is a predefined tolerance. Coefficients satisfying this criterion are snapped to their nearest simple value, producing the final interpretable expression.

\section{Experiments and results}\label{sec:results}

\textbf{Benchmark}
We evaluate SMILE on the Symbolic Regression Benchmark (SRBench)~\cite{lacava}, built on the Penn Machine Learning Benchmark (PMLB)~\cite{pmlb}, an open-source collection of machine learning datasets. The benchmark comprises 119 equations from the Feynman Lectures on Physics~\cite{aifeynman}, 14 ODE-based problems from the Strogatz database~\cite{strogatz}, and 57 black-box regression problems with unknown underlying equations. Following prior work, we restrict evaluation on the black-box problems to those with continuous features and input dimension $d \leq 10$. We compare against 15 methods included in the original SRBench evaluation, extended with PySR~\cite{sr-gp2023}, uDSR~\cite{udsr}, E2E~\cite{e2e}, and ParFam~\cite{parfam}. We also perform ablation studies to isolate the contribution of individual pipeline components. Details on the datasets and baselines are provided in Appendix~\ref{app:datasets} and~\ref{app:baselines}, respectively. Hyperparameter settings and computational resources are reported in Appendix~\ref{app:hyperparams}. We exclude EQL$^{\div}$~\cite{eql2018} from our comparison as its restricted operator set prevents it from expressing a substantial portion of the benchmark equations, making a direct comparison on the full SRBench unfair (see Appendix~\ref{app:baselines} for details).


\textbf{Metrics}
Following the SRBench evaluation protocol~\cite{lacava}, we report two metrics for ground-truth problems: the symbolic solution rate (SSR), defined as the percentage of equations exactly recovered by an algorithm, and the accuracy solution rate, defined as the percentage of problems achieving $R^2 > 0.999$, where $R^2 = 1 - \sum_{i=1}^{N}(y_i - \hat{y}_i)^2 / \sum_{i=1}^{N}(y_i - \bar{y})^2$, $\hat{y}_i$ is the model prediction, and $\bar{y}$ is the mean of the target values. All results are averaged over three independent trials. We additionally evaluate under noise by adding $\epsilon_i \sim \mathcal{N}(0, \sigma^2 \frac{1}{N}\sum_{i=1}^{N} y_i^2)$ to the targets, where $\sigma$ denotes the noise level. For the black-box problems, where ground-truth expressions are unavailable, we report the median $R^2$ and median complexity as defined in~\cite{lacava}.

\begin{figure}[htb]
\centering
\includegraphics[width=\linewidth]{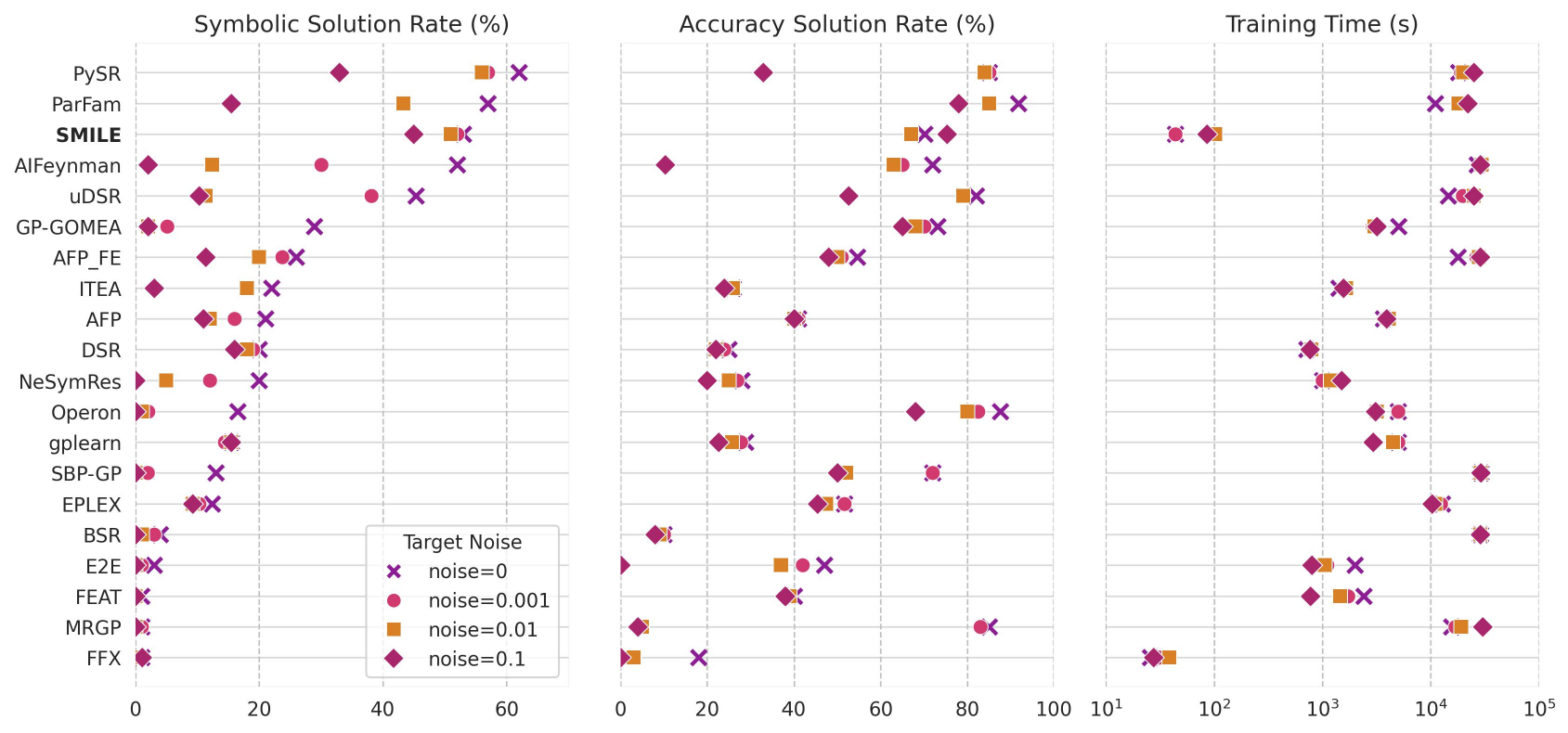}
\caption{Comparison of SMILE and baselines on the Feynman dataset in terms of SSR, accuracy solution rate ($R^2 > 0.999$), and training time.}
\label{fig:ssr-asr-time-fey}
\end{figure}

\begin{figure}[bt]
\centering
\includegraphics[width=0.9\linewidth]{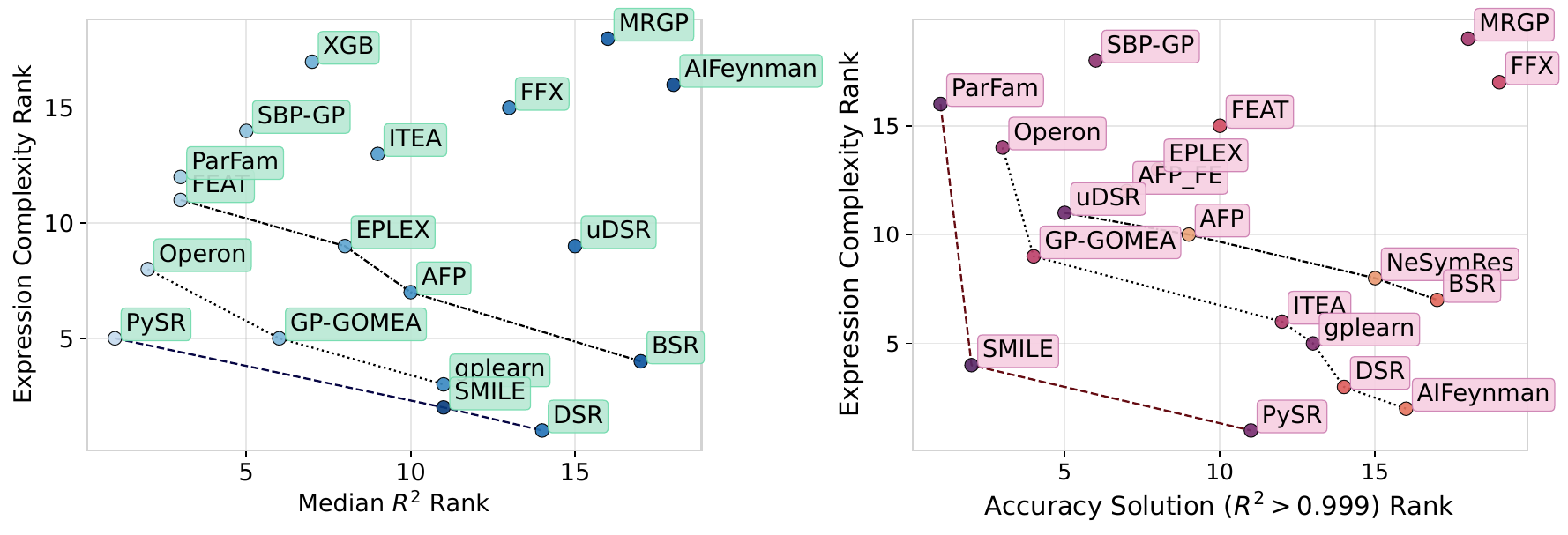}
\caption{Pareto front of expression complexity versus accuracy on the Feynman dataset (right, accuracy solution rate) and the black-box dataset (left, median $R^2$).}
\label{fig:pareto-com}
\end{figure}

\textbf{Ground-truth results}
\autoref{fig:ssr-asr-time-fey} compares SMILE against all baselines on the Feynman dataset with no noise and noise levels $\sigma \in \{0.001, 0.01, 0.1\}$ in terms of SSR, accuracy solution rate, and training time (s). On noise-free data, SMILE ranks third in SSR behind PySR and ParFam. However, SMILE is significantly more robust to noise: as the noise level increases, SMILE's SSR decreases by at most 8\%, compared to more than 30\% for ParFam and PySR, making SMILE the \textit{top-performing method at the highest noise level}. Regarding accuracy solution rate ($R^2 > 0.999$), SMILE performs lower than the top baselines. This is because the pipeline operates with limited depth to maintain a tractable pruning space, favoring exact symbolic recovery over regression accuracy. In terms of the training time, SMILE is significantly faster, recovering each expression in the order of minutes compared to several hours for ParFam, PySR, and AI Feynman. The Pareto front on the Feynman dataset (\autoref{fig:pareto-com}, right) shows that SMILE lies on the Pareto front alongside ParFam and PySR, while producing expressions of considerably lower complexity. On the Strogatz dataset (\autoref{fig:ssr-fey-stro}), SMILE achieves the second highest SSR after PySR on noise-free data, while being more robust to noise across all noise levels. Additional results, including median $R^2$ at increasing precision levels, accuracy solution rate comparisons between the Feynman and Strogatz datasets, and per-problem analysis are provided in Appendix~\ref{app:groundtruth}.

\begin{figure}[htb]
\centering
\includegraphics[width=0.8\linewidth]{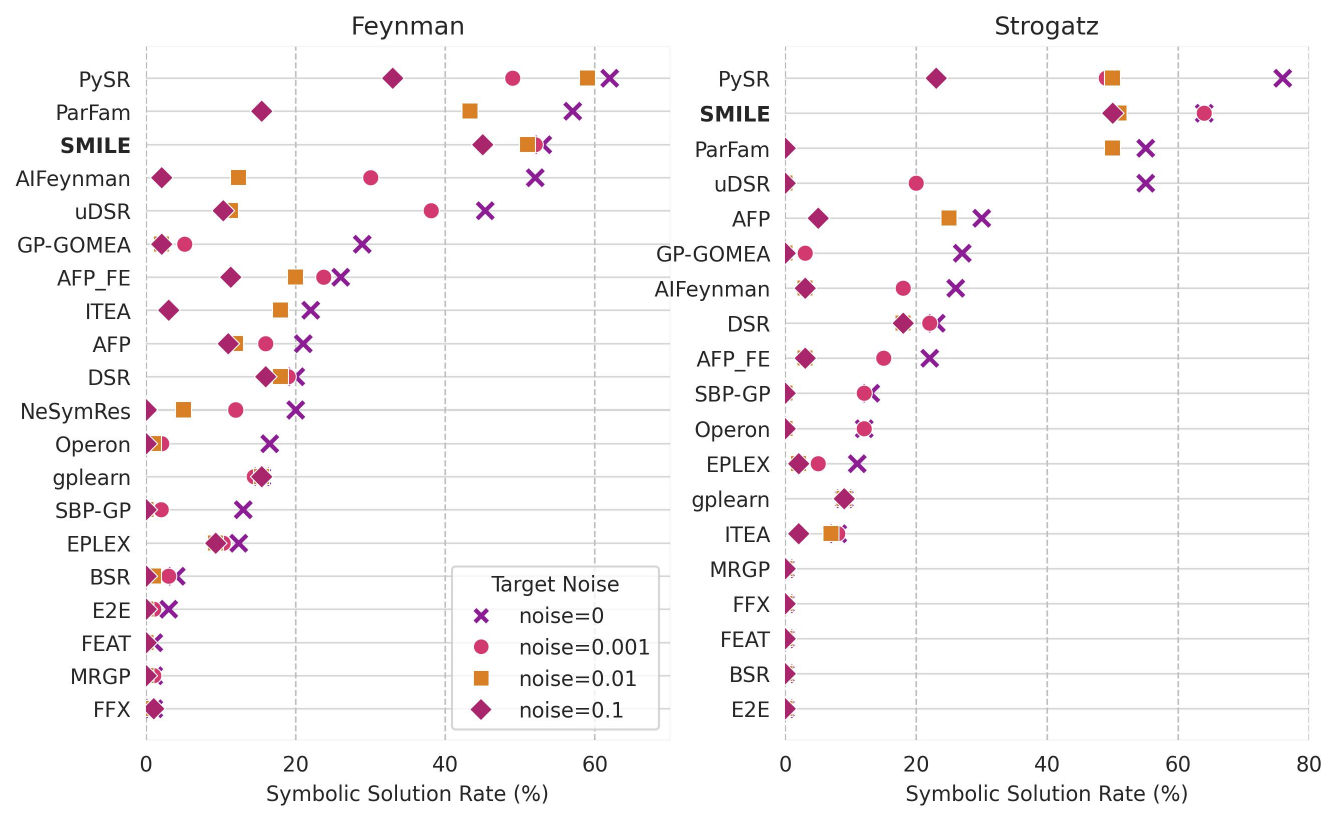}
\caption{SSR on the Feynman and Strogatz datasets across all evaluated baselines.}
\label{fig:ssr-fey-stro}
\end{figure}

\textbf{Black-box results}
The 57 black-box problems from PMLB span diverse real-world regression tasks where no ground-truth expression is available, so we report the median $R^2$ and expression complexity following \cite{lacava}. As shown in the left panel of \autoref{fig:pareto-com}, SMILE lies on the Pareto front of median $R^2$ versus complexity, achieving competitive accuracy while producing expressions of significantly lower complexity. Across both the Feynman and black-box datasets, SMILE consistently appears on the Pareto front, confirming that it maintains a favorable trade-off between accuracy and expression complexity regardless of whether ground-truth expressions are available. Detailed comparisons of median $R^2$, complexity, and training time are provided in Appendix \ref{app:blackbox}.

To evaluate the contribution of each pipeline component, we remove one component at a time and re-run the pipeline on the Feynman dataset. Removing greedy pruning or parametric optimization causes approximately 35\% SSR reduction, as remaining coefficients overwhelm downstream recovery stages. Removing variable analysis or gradient-based rounding reduces SSR by approximately 15\% with increased complexity. Detailed analysis is provided in Appendix~\ref{app:ablation}.

\section{Discussion and conclusion}

In this work, we presented SMILE, a hybrid SR framework built on a fixed set of interpretable activations, with an end-to-end pipeline integrating continuous optimization with discrete symbolic recovery. Our experiments on SRBench demonstrate several key strengths: SMILE ranks among the top methods in SSR across all noise levels, and achieves the highest SSR at the largest noise level, where competing methods degrade substantially. This robustness stems from the discrete recovery components — pruning and rounding — which naturally discard small or unstable parameters likely arising from noise, combined with the structural analysis stage that identifies compositional patterns. SMILE lies on the Pareto front of accuracy versus complexity across both ground-truth and black-box datasets, recovering significantly simpler expressions in a fraction of the time required by leading baselines, as validated through ablation studies. A current limitation is that SMILE favors shallow architectures to maintain a tractable pruning space. While increasing depth improves accuracy and $R^2$, it does not improve SSR, as the added degrees of freedom increase expression complexity and make pruning significantly harder. Additionally, the data-driven analysis assumes compositional patterns common in physical and mathematical laws, which may not hold for arbitrary real-world relationships. Future work could incorporate adaptive depth selection and extend the analysis stage to handle more general compositional forms.

\newpage
\bibliographystyle{unsrt}
\bibliography{references.bib}

@article{physics2023,
  title={A computational framework for physics-informed symbolic regression with straightforward integration of domain knowledge},
  author={Keren, Liron Simon and Liberzon, Alex and Lazebnik, Teddy},
  journal={Scientific Reports},
  volume={13},
  number={1},
  pages={1249},
  year={2023},
  publisher={Nature Publishing Group UK London}
}

@article{marin2023,
  title={Deep learning and symbolic regression for discovering parametric equations},
  author={Zhang, Michael and Kim, Samuel and Lu, Peter Y and Solja{\v{c}}i{\'c}, Marin},
  journal={IEEE Transactions on Neural Networks and Learning Systems},
  year={2023},
  publisher={IEEE}
}

@article{sr-material,
  title={Symbolic regression in materials science},
  author={Wang, Yiqun and Wagner, Nicholas and Rondinelli, James M},
  journal={MRS communications},
  volume={9},
  number={3},
  pages={793--805},
  year={2019},
  publisher={Cambridge University Press}
}

@article{sr-dynamic,
  title={Discovering governing equations from data by sparse identification of nonlinear dynamical systems},
  author={Brunton, Steven L and Proctor, Joshua L and Kutz, J Nathan},
  journal={Proceedings of the national academy of sciences},
  volume={113},
  number={15},
  pages={3932--3937},
  year={2016},
  publisher={National Academy of Sciences}
}

@article{sr-pde,
  title={Data-driven discovery of partial differential equations},
  author={Rudy, Samuel H and Brunton, Steven L and Proctor, Joshua L and Kutz, J Nathan},
  journal={Science advances},
  volume={3},
  number={4},
  pages={e1602614},
  year={2017},
  publisher={American Association for the Advancement of Science}
}

@article{sr-astro,
  title={Deep symbolic regression for physics guided by units constraints: toward the automated discovery of physical laws},
  author={Tenachi, Wassim and Ibata, Rodrigo and Diakogiannis, Foivos I},
  journal={The Astrophysical Journal},
  volume={959},
  number={2},
  pages={99},
  year={2023},
  publisher={IOP Publishing}
}

@article{cranmer-interp,
  title={Discovering symbolic models from deep learning with inductive biases},
  author={Cranmer, Miles and Sanchez Gonzalez, Alvaro and Battaglia, Peter and Xu, Rui and Cranmer, Kyle and Spergel, David and Ho, Shirley},
  journal={Advances in neural information processing systems},
  volume={33},
  pages={17429--17442},
  year={2020}
}

@article{aifeynman,
  title={AI Feynman: A physics-inspired method for symbolic regression},
  author={Udrescu, Silviu-Marian and Tegmark, Max},
  journal={Science advances},
  volume={6},
  number={16},
  pages={eaay2631},
  year={2020},
  publisher={American Association for the Advancement of Science}
}

@article{lacava,
  title={Contemporary symbolic regression methods and their relative performance},
  author={La Cava, William and Burlacu, Bogdan and Virgolin, Marco and Kommenda, Michael and Orzechowski, Patryk and de Fran{\c{c}}a, Fabr{\'\i}cio Olivetti and Jin, Ying and Moore, Jason H},
  journal={Advances in neural information processing systems},
  volume={2021},
  number={DB1},
  pages={1},
  year={2021}
}

@inproceedings{sr-gp2022,
  title={Taylor genetic programming for symbolic regression},
  author={He, Baihe and Lu, Qiang and Yang, Qingyun and Luo, Jake and Wang, Zhiguang},
  booktitle={Proceedings of the genetic and evolutionary computation conference},
  pages={946--954},
  year={2022}
}

@article{sr-gp2023,
  title={Interpretable machine learning for science with PySR and SymbolicRegression. jl},
  author={Cranmer, Miles},
  journal={arXiv preprint arXiv:2305.01582},
  year={2023}
}

@article{sr-trans2024,
  title={Symformer: End-to-end symbolic regression using transformer-based architecture},
  author={Vastl, Martin and Kulh{\'a}nek, Jon{\'a}{\v{s}} and Kubal{\'\i}k, Ji{\v{r}}{\'\i} and Derner, Erik and Babu{\v{s}}ka, Robert},
  journal={IEEE Access},
  year={2024},
  publisher={IEEE}
}

@article{sr-trans2023,
  title={Transformer-based planning for symbolic regression},
  author={Shojaee, Parshin and Meidani, Kazem and Barati Farimani, Amir and Reddy, Chandan},
  journal={Advances in Neural Information Processing Systems},
  volume={36},
  pages={45907--45919},
  year={2023}
}

@inproceedings{sr-gen2023,
  title={Deep generative symbolic regression with monte-carlo-tree-search},
  author={Kamienny, Pierre-Alexandre and Lample, Guillaume and Lamprier, Sylvain and Virgolin, Marco},
  booktitle={International Conference on Machine Learning},
  pages={15655--15668},
  year={2023},
  organization={PMLR}
}

@article{koza-genetic1994,
  title={Genetic programming as a means for programming computers by natural selection},
  author={Koza, John R},
  journal={Statistics and computing},
  volume={4},
  pages={87--112},
  year={1994},
  publisher={Springer}
}

@article{sr-dl2019,
  title={Deep symbolic regression: Recovering mathematical expressions from data via risk-seeking policy gradients},
  author={Petersen, Brenden K and Landajuela, Mikel and Mundhenk, T Nathan and Santiago, Claudio P and Kim, Soo K and Kim, Joanne T},
  journal={arXiv preprint arXiv:1912.04871},
  year={2019}
}

@article{biggio-seq2seq2020,
  title={A seq2seq approach to symbolic regression},
  author={Biggio, Luca and Bendinelli, Tommaso and Lucchi, Aurelien and Parascandolo, Giambattista},
  journal={Learning Meets Combinatorial Algorithms at NeurIPS2020},
  year={2020}
}

@inproceedings{eql2018,
  title={Learning equations for extrapolation and control},
  author={Sahoo, Subham and Lampert, Christoph and Martius, Georg},
  booktitle={International Conference on Machine Learning},
  pages={4442--4450},
  year={2018},
  organization={Pmlr}
}

@article{sindy,
  title={Discovering governing equations from data by sparse identification of nonlinear dynamical systems},
  author={Brunton, Steven L and Proctor, Joshua L and Kutz, J Nathan},
  journal={Proceedings of the national academy of sciences},
  volume={113},
  number={15},
  pages={3932--3937},
  year={2016},
  publisher={National Academy of Sciences}
}

@article{rl-spars2023,
  title={Efficient symbolic policy learning with differentiable symbolic expression},
  author={Guo, Jiaming and Zhang, Rui and Peng, Shaohui and Yi, Qi and Hu, Xing and Chen, Ruizhi and Du, Zidong and Li, Ling and Guo, Qi and Chen, Yunji and others},
  journal={Advances in Neural Information Processing Systems},
  volume={36},
  pages={36278--36304},
  year={2023}
}

@article{rl-spars2025, title={Noise-Resilient Symbolic Regression with Dynamic Gating Reinforcement Learning}, volume={39}, number={19}, journal={Proceedings of the AAAI Conference on Artificial Intelligence}, author={Sun, Chenglu and Shen, Shuo and Tao, Wenzhi and Xue, Deyi and Zhou, Zixia}, year={2025}, pages={20690-20698} }

@inproceedings{sr-mcts2022,
  title={Symbolic physics learner: Discovering governing equations via monte carlo tree search},
  author={Sun, Fangzheng and Liu, Yang and Wang, Jian-Xun and Sun, Hao},
  booktitle={International Conference on Learning Representations},
  year={2022}
}

@inproceedings{rag-sr,
  title={RAG-SR: Retrieval-augmented generation for neural symbolic regression},
  author={Zhang, Hengzhe and Chen, Qi and Banzhaf, Wolfgang and Zhang, Mengjie and others},
  booktitle={The Thirteenth International Conference on Learning Representations},
  year={2025}
}

@inproceedings{srscience2026,
  title={Sr-scientist: Scientific equation discovery with agentic ai},
  author={Xia, Shijie and Sun, Yuhan and Liu, Pengfei},
  booktitle={International Conference on Learning Representations},
  year={2026}
}

@inproceedings{egg-sr,
  title={EGG-SR: Embedding Symbolic Equivalence into Symbolic Regression via Equality Graph},
  author={Jiang, Nan and Wang, Ziyi and Xue, Yexiang},
  booktitle={International Conference on Learning Representations},
  year={2026}
}

@article{e2e,
  title={End-to-end symbolic regression with transformers},
  author={Kamienny, Pierre-Alexandre and d'Ascoli, St{\'e}phane and Lample, Guillaume and Charton, Fran{\c{c}}ois},
  journal={Advances in Neural Information Processing Systems},
  volume={35},
  pages={10269--10281},
  year={2022}
}

@article{lasr,
  title={Symbolic regression with a learned concept library},
  author={Grayeli, Arya and Sehgal, Atharva and Costilla-Reyes, Omar and Cranmer, Miles and Chaudhuri, Swarat},
  journal={Advances in Neural Information Processing Systems},
  volume={37},
  pages={44678--44709},
  year={2024}
}

@inproceedings{llm-sr,
  title={Llm-sr: Scientific equation discovery via programming with large language models},
  author={Shojaee, Parshin and Meidani, Kazem and Gupta, Shashank and Farimani, Amir Barati and Reddy, Chandan K},
  booktitle={International Conference on Learning Representations},
  year={2025}
}

@inproceedings{llm-srbench,
  title={Llm-srbench: A new benchmark for scientific equation discovery with large language models},
  author={Shojaee, Parshin and Nguyen, Ngoc-Hieu and Meidani, Kazem and Farimani, Amir Barati and Doan, Khoa D and Reddy, Chandan K},
  booktitle={International Conference on Machine Learning},
  year={2025}
}

@article{udsr,
  title={A unified framework for deep symbolic regression},
  author={Landajuela, Mikel and Lee, Chak Shing and Yang, Jiachen and Glatt, Ruben and Santiago, Claudio P and Aravena, Ignacio and Mundhenk, Terrell and Mulcahy, Garrett and Petersen, Brenden K},
  journal={Advances in Neural Information Processing Systems},
  volume={35},
  pages={33985--33998},
  year={2022}
}

@inproceedings{parfam,
  title={ParFam--(Neural Guided) Symbolic Regression via Continuous Global Optimization},
  author={Scholl, Philipp and Bieker, Katharina and Hauger, Hillary and Kutyniok, Gitta},
  booktitle={The Thirteenth International Conference on Learning Representations},
  year={2025}
}

@article{universal,
  title={Approximation by superpositions of a sigmoidal function},
  author={Cybenko, George},
  journal={Mathematics of control, signals and systems},
  volume={2},
  number={4},
  pages={303--314},
  year={1989},
  publisher={Springer}
}

@article{universal1,
  title={Approximation capabilities of multilayer feedforward networks},
  author={Hornik, Kurt},
  journal={Neural Networks},
  volume={4},
  number={2},
  pages={251--257},
  year={1991}
}

@inproceedings{gate2018,
  title={Learning sparse neural networks through $ L\_0 $ regularization},
  author={Louizos, Christos and Welling, Max and Kingma, Diederik P},
  booktitle={The Sixth International Conference on Learning Representations},
  year={2018}
}

@article{dc-sr,
  title={Divide and conquer: A quick scheme for symbolic regression},
  author={Luo, Changtong and Chen, Chen and Jiang, Zonglin},
  journal={International Journal of Computational Methods},
  volume={19},
  number={08},
  pages={2142002},
  year={2022}
}

@article{pmlb,
  title={PMLB: a large benchmark suite for machine learning evaluation and comparison},
  author={Olson, Randal S and La Cava, William and Orzechowski, Patryk and Urbanowicz, Ryan J and Moore, Jason H},
  journal={BioData mining},
  volume={10},
  number={1},
  pages={36},
  year={2017},
  publisher={Springer}
}

@article{strogatz,
  title={Inference of compact nonlinear dynamic models by epigenetic local search},
  author={La Cava, William and Danai, Kourosh and Spector, Lee},
  journal={Engineering Applications of Artificial Intelligence},
  volume={55},
  pages={292--306},
  year={2016},
  publisher={Elsevier}
}

@article{jin2019,
  title={Bayesian symbolic regression},
  author={Jin, Ying and Fu, Weilin and Kang, Jian and Guo, Jiadong and Guo, Jian},
  journal={arXiv preprint arXiv:1910.08892},
  year={2019}
}

\newpage
\appendix
\section*{Appendix}
\setcounter{theorem}{0}

\section{Proof of universal approximation}
\label{app:universality}

\begin{theorem}
SMILE networks are universal approximators.
\end{theorem}
\begin{proof}
We establish universality by showing that an SMILE network can emulate any standard multilayer perceptron (MLP) with one or more hidden layers of arbitrary width and sigmoid activations.

\textbf{Sigmoid construction}
We show that a single SMILE block can compute the sigmoid function $\sigma(a) = 1/(1 + \exp(-a))$. As illustrated in \autoref{fig:sig}a, the construction proceeds as follows:
\begin{enumerate}[nosep, leftmargin=*]
    \item The identity neuron computes the linear combination $a = \mathbf{w}^\top \mathbf{x}$.
    \item The exponential neuron computes $\exp(-a)$.
    \item The identity neuron adds the constant $1$ (provided by the appended input), yielding $1 + \exp(-a)$.
    \item The logarithm neuron computes $\log(1 + \exp(-a))$.
    \item The exponential neuron with a weight of $-1$ computes $\exp(-\log(1 + \exp(-a))) = \frac{1}{1 + \exp(-a)} = \sigma(a)$.
\end{enumerate}
We refer to this configuration as a \emph{Sigmoid block}.

\textbf{ MLP emulation.}
By stacking $L$ Sigmoid blocks as shown in \autoref{fig:sig}b, with identity neurons handling the intermediate affine transformations, we recover an $L$-layer MLP with sigmoid activations and arbitrary width (by replicating blocks in parallel). Since MLPs with sigmoid activations and at least one hidden layer of sufficient width are universal approximators~\cite{universal,universal1}, it follows that SMILE networks are also universal approximators.

\textbf{Polynomial emulation}
Additionally, universal approximability can be established through an alternative argument: the multiplication neuron, combined with the identity neuron, enables SMILE to represent any monomial $x_1^{a_1} x_2^{a_2} \cdots x_d^{a_d}$ in a single layer, and linear combinations of such terms via the identity neurons recover arbitrary polynomial functions. Since polynomials are dense in the space of continuous functions on compact sets (by the Stone--Weierstrass theorem), this provides a second, independent proof of universal approximation.
\end{proof}

\section{Pipeline details}\label{app:pipeline}
The complete SMILE discovery pipeline is formalized in Algorithms~\ref{alg:stage1} and~\ref{alg:stage2}. Algorithm~\ref{alg:stage1} describes the data analysis and decomposition stage, which determines whether the problem can be solved directly or requires variable decomposition, and selects the appropriate network depth accordingly. Algorithm~\ref{alg:stage2} describes the symbolic recovery stage, which distills the trained network into a closed-form expression through greedy pruning, parametric optimization, and gradient-based rounding.

\begin{algorithm}[htb]
\caption{SMILE Stage 1: Variable Analysis and Decomposition}\label{alg:stage1}
\begin{algorithmic}[1]
\REQUIRE Dataset $D = \{(\mathbf{x}_i, y_i)\}_{i=1}^{n}$, tolerance $\tau_{\text{window}}$
\ENSURE Trained SMILE network(s)
\FOR{each input variable $x_k$}
    \STATE Regress $\log|y|$ against $\log|x_k|$ and record $R^2_k$
\ENDFOR
\IF{all $R^2_k$ are comparably high}
    \STATE Compare polynomial fits to $y$ and $1/y$
    \IF{inverted fit is better}
        \STATE Set target $\leftarrow 1/y$
    \ELSE
        \STATE Set target $\leftarrow y$
    \ENDIF
    \STATE Train SMILE with 1 hidden layer $\rightarrow$ apply Algorithm~\ref{alg:stage2}
    \IF{exact symbolic solution found}
        \RETURN $\hat{f}$
    \ENDIF
    \STATE Train SMILE with 2 hidden layers $\rightarrow$ apply Algorithm~\ref{alg:stage2}
    \RETURN $\hat{f}$
\ELSE
    \STATE Identify variable $x_k$ with lowest $R^2_k$
    \STATE Select window $[x^* - \delta, x^* + \delta]$ satisfying $\left|\frac{\partial y}{\partial x_k}\right| \cdot 2\delta \, / \, \mathrm{std}(y) < \tau_{\text{window}}$
    \STATE Train SMILE with 1 hidden layer on windowed data to recover $f(\mathbf{x}_{\setminus k})$
    \STATE Compute residuals at two fixation points
    \IF{multiplicative residuals are approximately constant}
        \STATE Set $r \leftarrow y / f(\mathbf{x}_{\setminus k})$
    \ELSE
        \STATE Set $r \leftarrow y - f(\mathbf{x}_{\setminus k})$
    \ENDIF
    \STATE Train SMILE with 1 hidden layer on residuals to recover $g(x_k)$
    \STATE Combine: $\hat{f} \leftarrow f \cdot g$ or $\hat{f} \leftarrow f + g$
    \STATE Apply Algorithm~\ref{alg:stage2} to $\hat{f}$
    \RETURN $\hat{f}$
\ENDIF
\end{algorithmic}
\end{algorithm}

\begin{algorithm}[hbt]
\caption{SMILE Stage 2: Symbolic Recovery}\label{alg:stage2}
\begin{algorithmic}[1]
\REQUIRE Trained SMILE network, tolerances $\tau_{\text{prune}}$, $\tau_{\text{round}}$
\ENSURE Closed-form symbolic expression $\hat{f}$
\STATE \textit{Greedy pruning:}
\REPEAT
    \STATE Identify gate $g$ with smallest impact on training $R^2$
    \STATE Force $\sigma(g) \approx 0$
\UNTIL{no gate can be closed without $R^2$ drop $> \tau_{\text{prune}}$}
\STATE Extract closed-form expression from remaining active subgraph
\STATE \textit{Parametric optimization:}
\STATE Refit all numerical constants via least-squares optimization
\STATE \textit{Gradient-based rounding:}
\FOR{each coefficient $c_i$}
    \IF{$|c_i - r_i| \cdot \left|\frac{\partial \hat{f}}{\partial c_i}\right| < \tau_{\text{round}}$ for all $\mathbf{x} \in S$}
        \STATE Snap $c_i \leftarrow r_i$
    \ENDIF
\ENDFOR
\RETURN $\hat{f}$
\end{algorithmic}
\end{algorithm}

\subsection{Variable analysis: Window selection algorithm}\label{app:varfix}

The window selection criterion ensures that fixing $x_k$ within a window $[x^* - \delta, x^* + \delta]$ introduces negligible variation in the target $y$. The starting requirement is that the relative change in $y$ due to $x_k$ remains small:
\begin{equation}\label{eq:window-exact}
\frac{|y(\mathbf{x}_{\setminus k}, x^* + \delta) - y(\mathbf{x}_{\setminus k}, x^* - \delta)|}{\mathrm{std}(y)} < \tau.
\end{equation}
By a first-order Taylor expansion around $x^*$:
\begin{equation}
y(\mathbf{x}_{\setminus k}, x^* + \delta) \approx y(\mathbf{x}_{\setminus k}, x^*) + \frac{\partial y}{\partial x_k} \cdot \delta, \qquad y(\mathbf{x}_{\setminus k}, x^* - \delta) \approx y(\mathbf{x}_{\setminus k}, x^*) - \frac{\partial y}{\partial x_k} \cdot \delta.
\end{equation}
Substituting into Equation~\ref{eq:window-exact} and simplifying:
\begin{equation}\label{eq:window-approx}
\frac{\left|\frac{\partial y}{\partial x_k}\right| \cdot 2\delta}{\mathrm{std}(y)} < \tau.
\end{equation}
In practice, we estimate the partial derivative empirically and define a constancy score that quantifies whether fixing $x_k$ is a defensible approximation.

\paragraph{Window sweep}
We generate $n_{\text{windows}} = 6$ candidate relative half-widths, log-spaced from a maximum percentage $p_{\max}$ (default 20\%) down to a minimum $p_{\min}$ (default 2\%):
\begin{equation}
p \in \texttt{logspace}\bigl(\log_{10} p_{\max},\; \log_{10} p_{\min},\; n_{\text{windows}}\bigr).
\end{equation}
Each $p$ defines a candidate window $[x^* - p \cdot x^*,\; x^* + p \cdot x^*]$ with width $W = 2p \cdot x^*$.

\paragraph{Constancy score}
For each candidate window, starting from the largest, we constrain the data to a slab where $x_k$ falls within the window and all other variables $\mathbf{x}_{\setminus k}$ are restricted to a narrow band so that only the effect of $x_k$ is measured. Within this range, we estimate the partial derivative by fitting a linear regression $y \approx \beta \, x_k + \text{const}$ and compute the constancy score:
\begin{equation}
C = \frac{|\beta| \cdot W}{\mathrm{std}(y_{\text{slab}})}.
\end{equation}
This directly operationalizes Equation~\ref{eq:window-approx}: $|\beta|$ approximates $\left|\frac{\partial y}{\partial x_k}\right|$, so the numerator estimates the total drift in $y$ as $x_k$ traverses the window, while the denominator captures the natural variability of $y$ within the slab. A window passes if $C < \tau_{window}$ (default $\tau_{window} = 0.05$), meaning the drift introduced by $x_k$ is at most 5\% of $y$'s standard deviation.

\paragraph{Selection and fallback}
Candidates are evaluated from largest to smallest. The first window that satisfies $C < \tau$ is selected, ensuring the widest defensible approximation and thus the most training data for the subsequent SMILE fit. If no window passes, the algorithm falls back to the smallest candidate $\pm p_{\min} \cdot x^*$, where the constant approximation is most defensible even if imperfect.

\subsection{Gradient-Based rounding}\label{app:rounding}

\begin{theorem}[Gradient-based Rounding]
\label{thm:grad-round}
Let $\hat{f} : \mathbb{R}^{p + d} \rightarrow \mathbb{R}$ be a function of $p$ parameters and an input $\mathbf{x} \in \mathbb{R}^d$, with continuous partial derivatives with respect to a parameter $P_i$. If
\[
\left\lvert h \cdot \frac{\partial \hat{f}(P_i; \mathbf{x})}{\partial P_i} \right\rvert < \epsilon \quad \forall P_i \in [\alpha, \alpha + h], \, \mathbf{x},
\]
for some small $\epsilon > 0$, then
\[
\hat{f}(P_i = \alpha + h; \mathbf{x}) \approx \hat{f}(P_i = \alpha; \mathbf{x}).
\]
\end{theorem}

\begin{proof}
Let $g_{\mathbf{x}}(y) = \hat{f}(P_i = y; \mathbf{x})$. By the mean value theorem, there exists $\alpha_0 \in [\alpha, \alpha+h]$ such that
\[
g_{\mathbf{x}}(\alpha+h) - g_{\mathbf{x}}(\alpha) = h \cdot g_{\mathbf{x}}'(\alpha_0).
\]
Thus,
\[
|g_{\mathbf{x}}(\alpha+h) - g_{\mathbf{x}}(\alpha)| = |h \cdot g_{\mathbf{x}}'(\alpha_0)| \leq h \cdot \max_{y \in [\alpha, \alpha+h]} |g_{\mathbf{x}}'(y)|,
\]
where the maximum exists due to the continuity of $g_{\mathbf{x}}'$ on the closed interval. If this upper bound is smaller than a predefined threshold $\tau$, i.e.,
\[
h \cdot \max_{y \in [\alpha, \alpha+h]} |g_{\mathbf{x}}'(y)| < \tau,
\]

The change in the function value is negligible and we may approximate $g_{\mathbf{x}}(\alpha+h) \approx g_{\mathbf{x}}(\alpha)$. If the condition holds for all $\mathbf{x} \in S$, where $S$ denotes the set of all input samples, then $\hat{f}(\alpha + h) \approx \hat{f}(\alpha)$.
\end{proof}

In practice, we set $\alpha = r_i$ (the candidate rounded value) and $h = c_i - r_i$ (the rounding offset). The change in the expression due to rounding is bounded by:
\[
|\hat{f}(c_i, \mathbf{x}) - \hat{f}(r_i, \mathbf{x})| \leq |c_i - r_i| \cdot \max_{s \in [r_i, c_i]} \left| \frac{\partial \hat{f}}{\partial c_i}(s, \mathbf{x}) \right|.
\]
To avoid computing gradients at multiple points, we approximate the maximum by evaluating the derivative at the rounded value $r_i$:
\[
|\hat{f}(c_i, \mathbf{x}) - \hat{f}(r_i, \mathbf{x})| \approx |c_i - r_i| \cdot \left| \frac{\partial \hat{f}}{\partial c_i}(r_i, \mathbf{x}) \right|.
\]
If this estimated change is smaller than a predefined threshold $\tau$ for all input samples $\mathbf{x} \in S$, we replace $c_i$ with the rounded value $r_i$, recovering exact symbolic constants without degrading the quality of the expression.

\section{Datasets}\label{app:datasets}

We evaluate SMILE on the Symbolic Regression Benchmark (SRBench)~\cite{lacava}, built on the Penn Machine Learning Benchmark (PMLB)~\cite{pmlb}. SRBench comprises 252 regression problems organized into three groups: 119 Feynman equations, 14 Strogatz ODE problems, and 122 black-box regression problems. We restrict the black-box evaluation to 57 problems with continuous features and input dimension $d \leq 10$, following prior work~\cite{lacava}. All datasets are publicly available through PMLB.\footnote{\url{https://github.com/EpistasisLab/pmlb}}

\textbf{Feynman dataset} This collection consists of 119 equations from the Feynman Symbolic Regression Database~\cite{aifeynman},\footnote{\url{https://space.mit.edu/home/tegmark/aifeynman.html}} spanning classical mechanics, electromagnetism, thermodynamics, and quantum mechanics. Of these, 100 are drawn from the Feynman Lectures on Physics~\cite{aifeynman} and serve as the core benchmark, covering expressions with 1 to 9 input variables and involving elementary functions such as polynomials, trigonometric functions, exponentials, and square roots. The remaining 19 are bonus equations selected from other seminal physics textbooks for their greater complexity and difficulty, involving deeper compositions and more intricate variable interactions. Each problem is provided with synthetically generated data sampled uniformly over a specified input domain. Each Feynman problem provides $10^6$ data points; in all experiments, we subsample $10^4$ points uniformly at random for training.

\textbf{Strogatz dataset.} This collection consists of 14 ordinary differential equation problems from the ODE-Strogatz repository~\cite{strogatz},\footnote{\url{https://github.com/lacava/ode-strogatz}} based on nonlinear dynamics and chaos. These problems require recovering the right-hand side of governing ODEs, with expressions that are generally compact but involve nonlinear interactions between variables. The dataset tests a method's ability to recover precise functional forms in low-dimensional dynamical settings.

\textbf{Black-box dataset.} The full PMLB collection contains 122 regression problems without known ground-truth expressions, sourced from various open-source repositories~\cite{pmlb}. Since no true underlying equations are available, these problems are evaluated using median $R^2$ and expression complexity rather than symbolic solution rate, testing whether methods can produce accurate yet simple models on data that may not exhibit the compositional structure typical of physical laws.

\section{Baselines}\label{app:baselines}

We compare SMILE against 19 methods evaluated on SRBench. Fourteen of these are included in the original SRBench evaluation~\cite{lacava}. We additionally compare against four recent methods that have achieved state-of-the-art performance on different experiments and we described them below.

\textbf{PySR}~\cite{sr-gp2023} is a high-performance GP framework built on the Julia library SymbolicRegression.jl. It employs a multi-population evolutionary algorithm with an evolve--simplify--optimize loop, where candidate expression trees are evolved through mutation and crossover, algebraically simplified, and have their constants optimized via gradient-based methods. Multiple populations are run in parallel with periodic migration to improve exploration.

\textbf{uDSR}~\cite{udsr} is a unified framework that integrates five SR solution strategies: recursive problem simplification, neural-guided search, large-scale pre-training, genetic programming, and linear models. These components are combined as connected but non-overlapping modules, with each strategy contributing complementary capabilities. While achieving strong performance, the framework involves substantial computational cost due to the coordination of multiple strategies.

\textbf{E2E}~\cite{e2e} is a transformer-based approach that directly predicts the full mathematical expression, including numerical constants, from a set of input-output observations in a single forward pass. The model is pretrained on a large corpus of synthetically generated equation-data pairs and refines the predicted constants via BFGS optimization. While inference is nearly instantaneous, the method is sensitive to its training distribution and may struggle on expressions outside the patterns seen during pretraining.

\textbf{ParFam}~\cite{parfam} reformulates SR as continuous global optimization over parametric families of symbolic functions. It iterates through candidate function families parameterized by compositions of elementary operations, optimizing their coefficients via basin-hopping with L-BFGS local search. A sparsity-promoting finetuning step progressively zeroes out small coefficients to recover compact expressions. The activation functions and polynomial degrees of the parametric families are user-specified.

\textbf{Note on EQL$^{\div}$} We do not include EQL$^{\div}$~\cite{eql2018} in our comparisons. Although SMILE shares the principle of embedding symbolic activations into a neural architecture, EQL$^{\div}$ lacks support for logarithm, exponential, and square root operations, preventing it from expressing a significant portion of the Feynman and Strogatz equations. Prior work~\cite{parfam} evaluated EQL$^{\div}$ on a reduced benchmark of 96 equations excluding these operations and reported a symbolic solution rate of 16.7\%, substantially below all competitive baselines. Given this limited expressivity, a full-benchmark comparison would be unfair, and we therefore follow the same exclusion rationale as~\cite{parfam}.

\section{Additional results}\label{app:results}

\subsection{Hyperparameters and Configuration}\label{app:hyperparams}

\autoref{tab:hyperparams} summarizes all hyperparameters and configuration settings used across all experiments. All values are fixed across all datasets with no per-problem tuning.

\begin{table}[H]
\centering
\caption{SMILE hyperparameters and pipeline configuration.}\label{tab:hyperparams}
\begin{tabular}{@{}lll@{}}
\toprule
\textbf{Category} & \textbf{Parameter} & \textbf{Value} \\
\midrule
\multirow{5}{*}{Architecture}
& Neurons per hidden layer & 5 (Sin, Mul, Identity, Log, Exp) \\
& Hidden layers & 1 or 2 \\
& Residual connections & Fully dense, all layers \\
& Bias & Learnable constant 1 appended to input \\
& Gate initialization & 0 ($\sigma = 0.5$) \\
\midrule
\multirow{6}{*}{Training}
& Optimizer & Adam \\
& Learning rate & 0.1 \\
& Epochs & 1000 \\
& Batch size & 512\\
& Early stopping patience & 300 \\
& Trials & 3\\
\midrule
\multirow{2}{*}{Clamping}
& $x_l$ (log threshold) & 0.005\\
& $x_e$ (exp threshold) & 4\\
\midrule
\multirow{3}{*}{Pipeline}
& $\tau_{\text{window}}$ (window tolerance) & 0.05 \\
& $\tau_{\text{prune}}$ (pruning tolerance) & 0.01\\
& $\tau_{\text{round}}$ (rounding tolerance) & 0.001\\
\bottomrule
\end{tabular}
\end{table}

\begin{table}[H]
\centering
\caption{Computational resources.}\label{tab:compute}
\begin{tabular}{@{}ll@{}}
\toprule
\textbf{Parameter} & \textbf{Value} \\
\midrule
CPU & Intel Xeon Gold 5420+ \\
GPU & NVIDIA RTX A5000 (24 GB) \\
Framework & PyTorch \\
Avg. time per problem & 43.9 s\\
\bottomrule
\end{tabular}
\end{table}

\newpage

\subsection{Ground-truth datasets}\label{app:groundtruth}

\autoref{fig:r2-fey} presents a detailed comparison of all baselines on the Feynman dataset in terms of median $R^2$ test at increasing precision levels, expression complexity, and training time. We define expression complexity as the number of nodes in the symbolic expression tree, where each operator (e.g., $\sin$, $\log$, $+$, $\times$), variable, and constant counts as a single node. For example, the expression $\sin(x_1) + x_2$ has a complexity of 4 ($\sin$, $x_1$, $+$, $x_2$). This definition is consistent with the complexity measure used in SRBench, enabling direct comparison across methods. SMILE and ParFam achieve comparable median $R^2$ at the highest precision levels, both producing accurate predictions across the dataset. In terms of expression complexity, SMILE, AI Feynman, and PySR produce the simplest expressions among all baselines, with SMILE consistently at the lowest complexity. Training time follows the same trend observed in the main text, with SMILE completing in the order of minutes while most competitive baselines require several hours. 

\begin{figure}[htb]
\centering
\includegraphics[width=\linewidth]{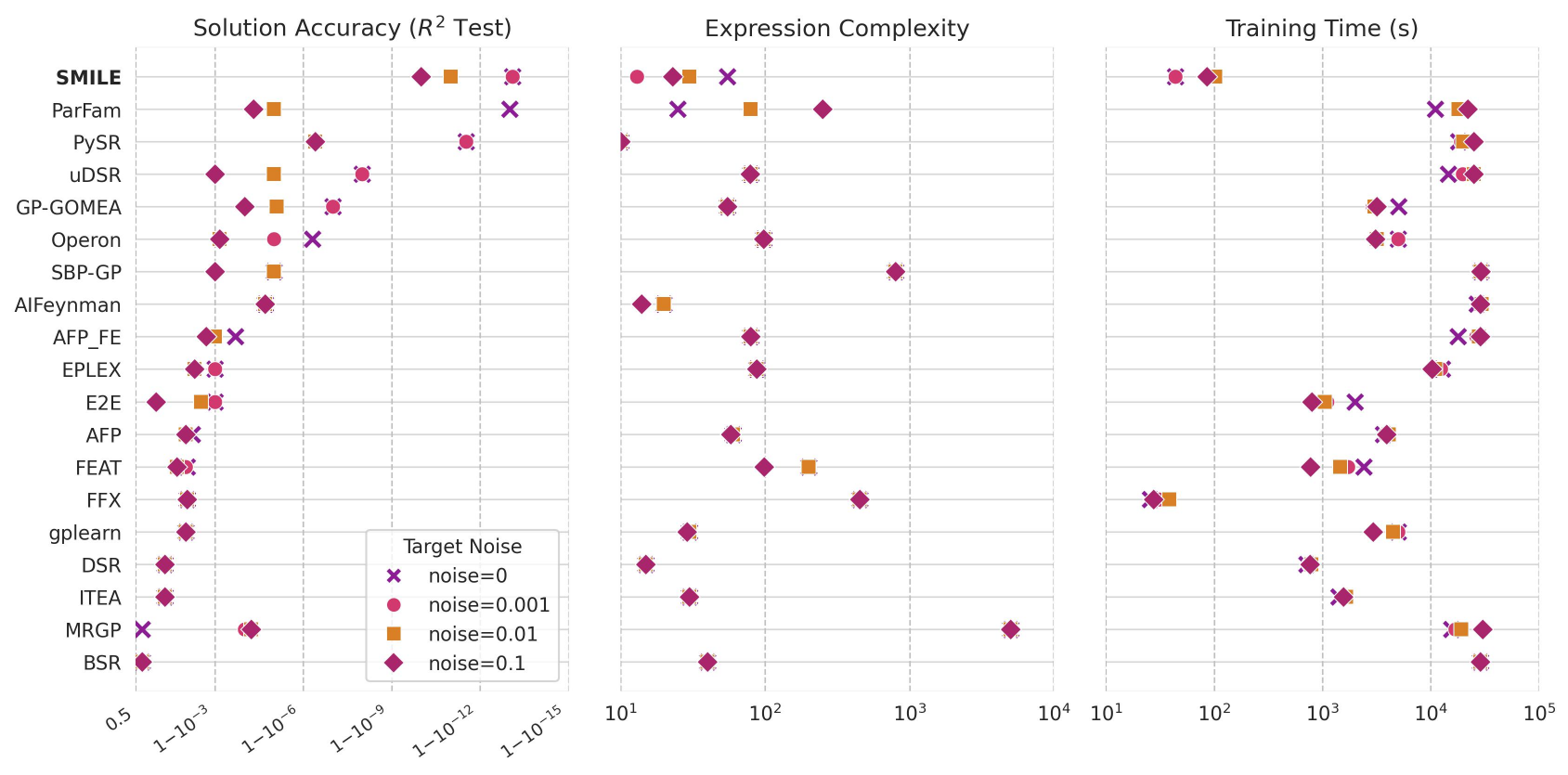}
\caption{Comparison of all methods on the Feynman dataset in terms of median $R^2$ test at multiple precision levels, expression complexity, and training time.}
\label{fig:r2-fey}
\end{figure}

\autoref{fig:asr-fey-stro} compares the accuracy solution rate ($R^2 > 0.999$) on the Feynman and Strogatz datasets across all baselines. As discussed in \autoref{sec:results}, SMILE's accuracy solution rate on the Feynman dataset is lower than the top baselines due to the shallow architecture favoring exact symbolic recovery. On the Strogatz dataset, SMILE achieves a higher accuracy solution rate than on Feynman, and demonstrates strong robustness to noise, decreasing by at most 20\% from the noise-free setting to the highest noise level, compared to approximately 70\% for ParFam and PySR. \textit{At the highest noise level, SMILE achieves the best accuracy solution rate among all baselines on the Strogatz dataset.}

\begin{figure}[htb]
\centering
\includegraphics[width=0.9\linewidth]{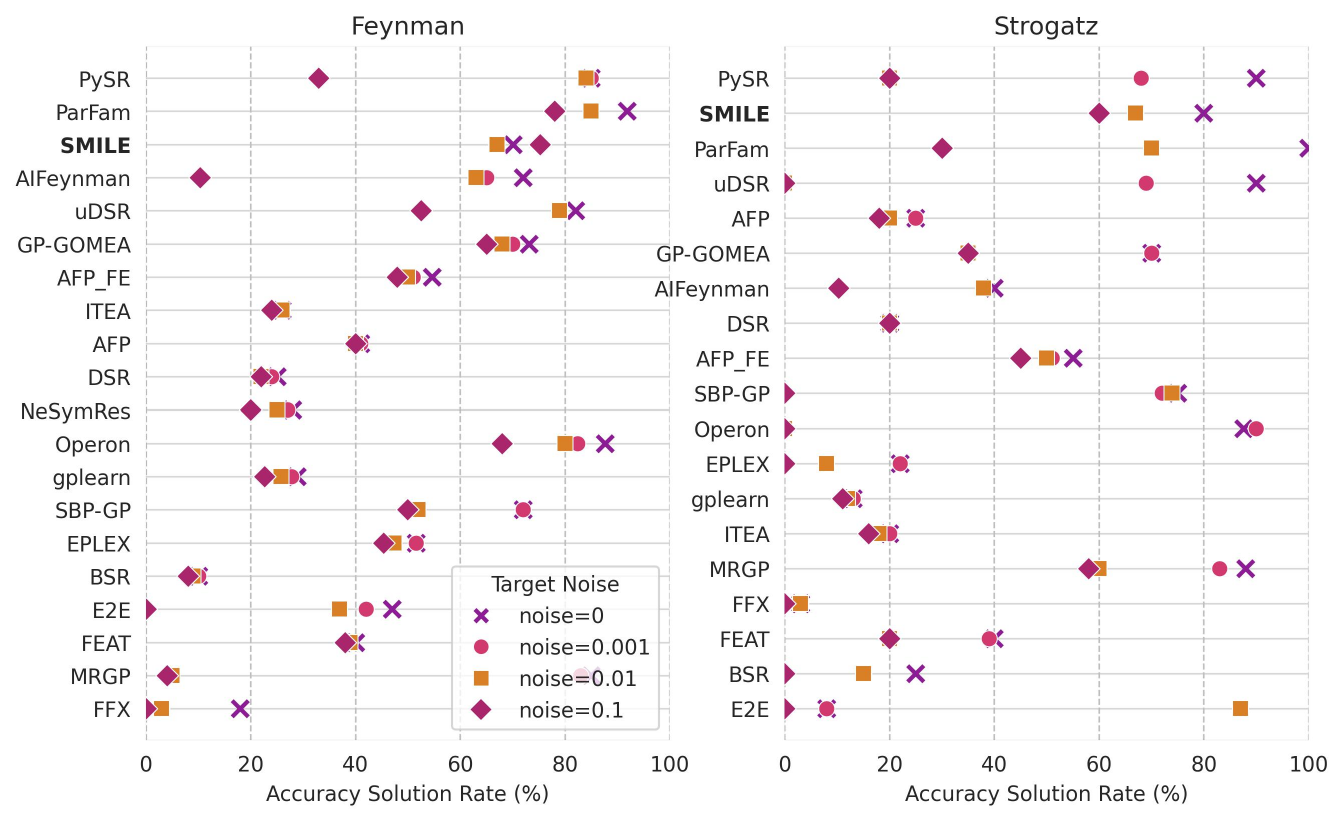}
\caption{Accuracy solution rate ($R^2 > 0.999$) on the Feynman and Strogatz datasets across all evaluated methods.}
\label{fig:asr-fey-stro}
\end{figure}

\autoref{tab:examples} presents representative examples of expressions recovered by SMILE on the Feynman dataset. The table includes formulas of varying complexity and number of input variables, illustrating SMILE's ability to recover exact symbolic forms across different noise levels. In several cases, SMILE recovers expressions that are structurally different from the ground truth but mathematically equivalent, such as replacing $\cos(\theta)$ with $\sin(\pi/2 - \theta)$ (II.15.4) or $\cos(\theta)\sin(\theta)$ with $\sin(2\theta)/2$ (II.6.15b), demonstrating that the recovery pipeline is not constrained to a single canonical form.

\begin{table}[H]
\centering
\caption{Examples of expressions recovered by SMILE on the Feynman dataset under noise-free and noisy conditions.}\label{tab:examples}
\renewcommand{\arraystretch}{2}
\resizebox{\textwidth}{!}{
\begin{tabular}{@{}llccccc@{}}
\toprule
\textbf{Problem} & \textbf{Ground truth} & \textbf{SMILE output} & $\boldsymbol{\sigma=0}$ & $\boldsymbol{\sigma=0.001}$ & $\boldsymbol{\sigma=0.01}$ & $\boldsymbol{\sigma=0.1}$ \\
\midrule
I.6.2a & $\dfrac{e^{-\theta^2/2}}{\sqrt{2\pi}}$ & $\dfrac{e^{-\theta^2/2}}{\sqrt{2\pi}}$ & \checkmark & $\times$  & $\times$  & $\times$  \\[6pt]
I.12.2 & $\dfrac{q_1 q_2 r}{4\pi \epsilon r^3}$ & $\dfrac{q_1 q_2}{4\pi \epsilon r^2}$ & \checkmark & \checkmark & \checkmark & \checkmark \\[6pt]
I.18.14 & $m r v \sin(\theta)$ & $m r v \sin(\theta)$ & \checkmark & \checkmark & \checkmark & $\times$ \\[6pt]
I.40.1 & $n_0 e^{-mgx/(k_b T)}$ & $n_0 e^{-mgx/(k_b T)}$ & \checkmark & $\times$ & $\times$ & $\times$ \\[6pt]
I.43.43 & $\dfrac{k_b v}{(\gamma-1)A}$ & $\dfrac{k_b v}{(\gamma-1)A}$ & \checkmark & \checkmark & $\times$ & $\times$ \\[6pt]
 II.15.4 & $-\mu B \cos(\theta)$ & $-\mu B \sin(\pi/2 - \theta)$ & \checkmark & \checkmark & \checkmark & \checkmark \\[6pt]
III.8.54 & $\sin^2\!\left(\dfrac{2\pi E_n t}{h}\right)$ & $\sin^2\!\left(\dfrac{2\pi E_n t}{h}\right)$ & \checkmark & \checkmark & \checkmark & $\times$ \\[6pt]
 III.4.32 & $\dfrac{1}{e^{h\omega/(2\pi k_b T)}-1}$ & $\dfrac{1}{e^{h\omega/(2\pi k_b T)}-1}$ & \checkmark & $\times$ & $\times$ & $\times$ \\[6pt]
II.6.15b & $\dfrac{3p_d \cos(\theta)\sin(\theta)}{4\pi\epsilon r^3}$ & $\dfrac{3p_d \sin(2\theta)}{8\pi\epsilon r^3}$ & \checkmark & \checkmark & \checkmark & \checkmark \\[6pt]
 I.47.23 & $\sqrt{\dfrac{\gamma \, p_r}{\rho}}$ & $\sqrt{\dfrac{\gamma \, p_r}{\rho}}$ & \checkmark & \checkmark & \checkmark & \checkmark \\[6pt]
II.6.11 & $\dfrac{p_d \cos(\theta)}{4\pi\epsilon r^2}$ & $\dfrac{-p_d \sin(\theta - \pi/2)}{4\pi\epsilon r^2}$ & \checkmark & \checkmark & $\times$ & $\times$ \\[6pt]
\bottomrule
\end{tabular}
}
\end{table}

To evaluate robustness to noise, we plot the SSR as a function of noise level for the top five methods on both the Feynman and Strogatz benchmarks (Figures~\ref{fig:no-robust}). On the Feynman benchmark, SMILE maintains competitive SSR across all noise levels. Notably, while other top-performing methods show sudden drops in SSR as noise increases, SMILE remains stable, demonstrating strong robustness to noisy observations. On the Strogatz benchmark, a similar trend is observed, with SMILE exhibiting a more gradual decline in SSR compared to the other top-performing methods (other methods achieve near 0\% SSR in high noise levels). These results suggest that the combination of continuous optimization and discrete symbolic recovery in SMILE provides a natural form of noise regularization, as the pruning and rounding stages discard small or unstable components that are more likely to arise from fitting noise rather than true structure.

\begin{figure}[htb]
\centering
\includegraphics[width=\linewidth]{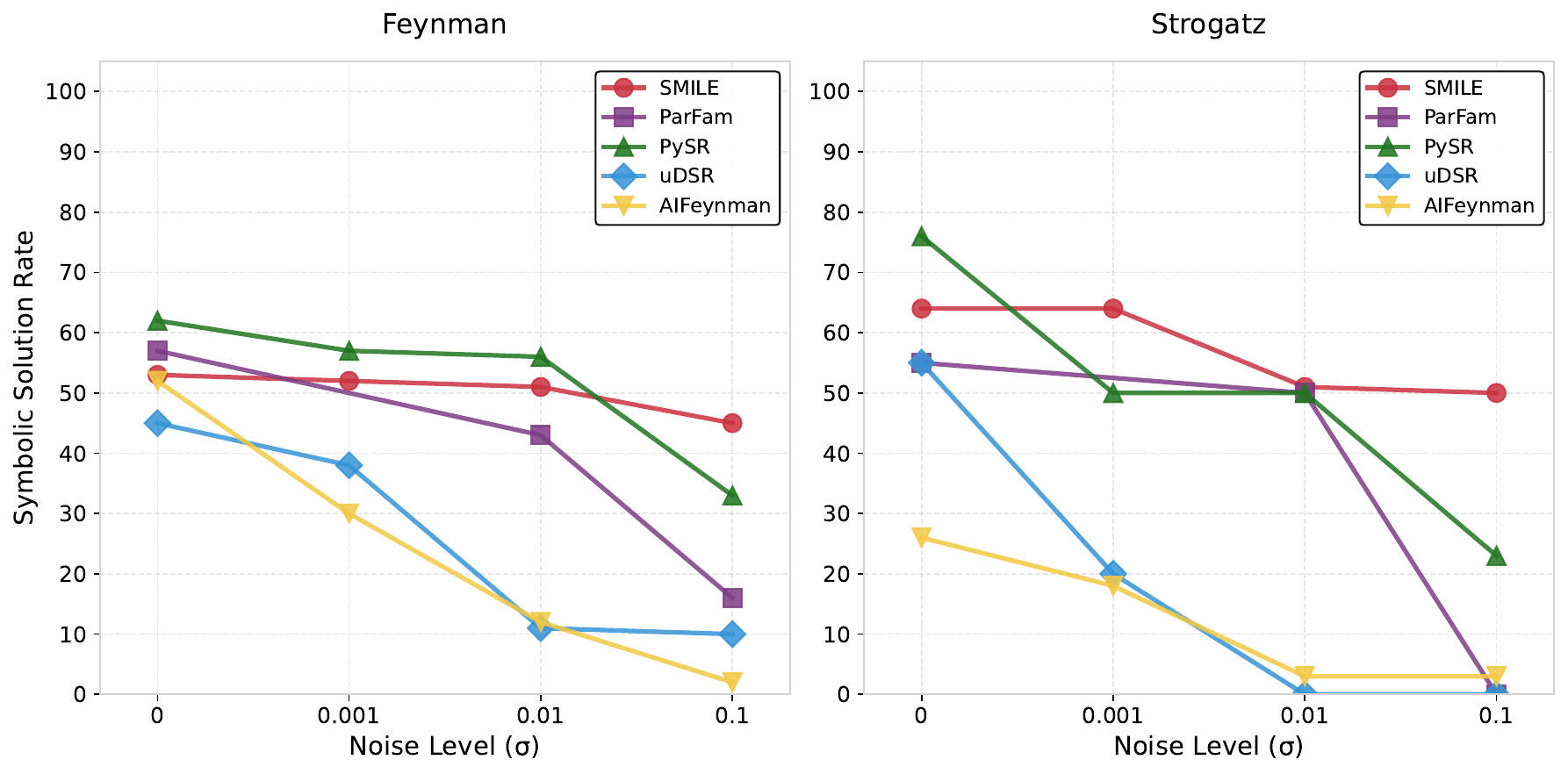}
\caption{SSR as a function of noise level for the top five methods on the Feynman and Strogatz benchmarks.}
\label{fig:no-robust}
\end{figure}

\paragraph{Qualitative visualization}
To provide an intuitive view of the symbolic expressions recovered by SMILE, we visualize six representative problems with two input variables in Figure~\ref{fig:heatmaps}. For each problem, we plot the ground-truth function (left) and SMILE's recovered expression (right) as heatmaps over the input domain, with training data points overlaid. The first three formulas show problems from the Jin et al.~\cite{jin2019} benchmark, while the second three formulas show problems from the Feynman benchmark. For problems with three input variables, one variable is held fixed to enable 2D visualization. In five of the six cases, the recovered expression matches the ground truth exactly ($R^2 = 1$). In the remaining case — the Gaussian probability density function — SMILE recovers a close approximation ($R^2 > 0.99$) that visually aligns with the ground truth but does not match the exact symbolic form.

\begin{figure}[htb]
\centering
\includegraphics[width=\linewidth]{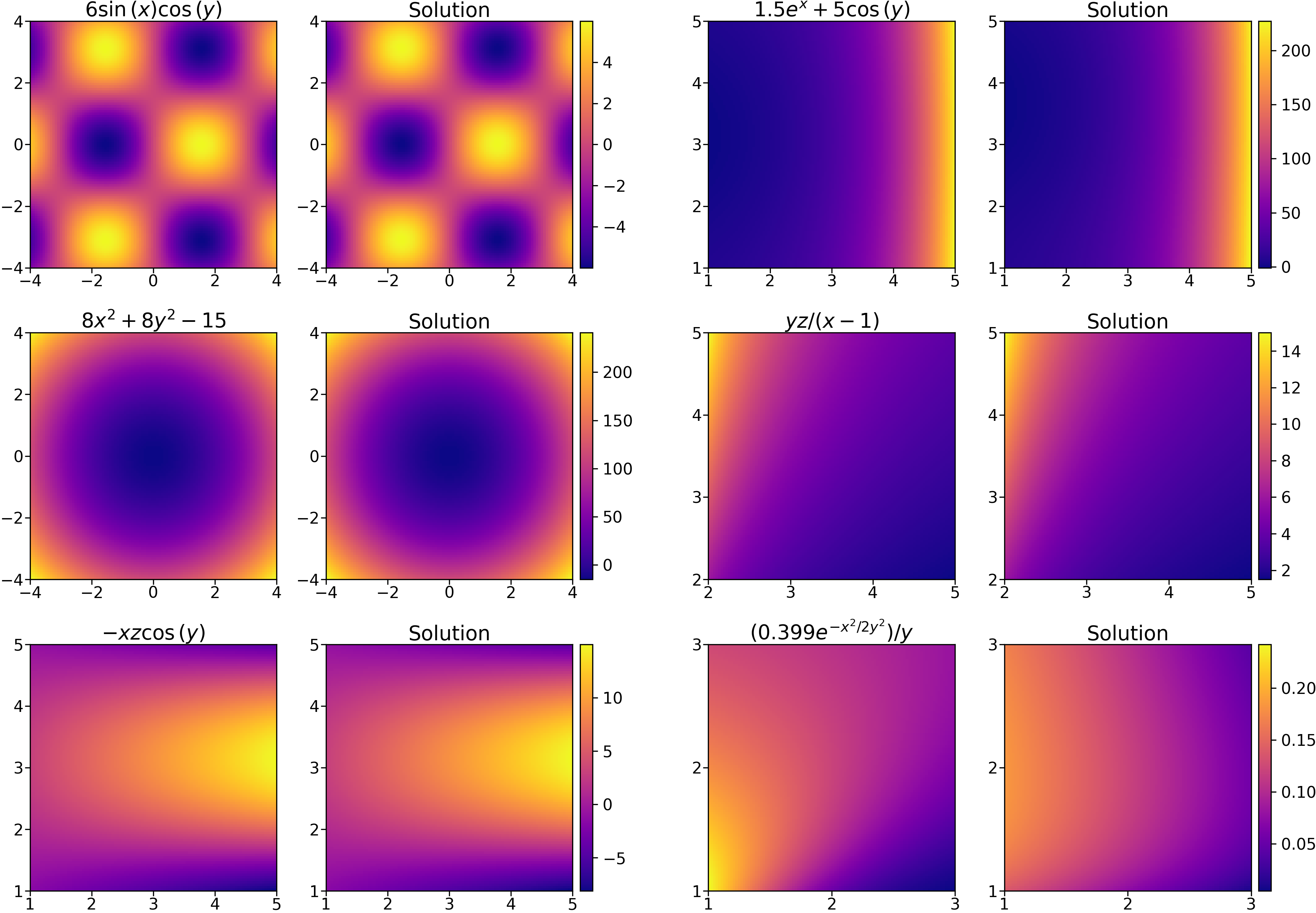}
\caption{Ground-truth (left) and SMILE-recovered (right) expressions for six problems: three from the Jin et al.~\cite{jin2019} benchmark and three from the Feynman benchmark. For three-variable problems, one variable is held fixed. Five expressions are recovered exactly ($R^2 = 1$); the Gaussian PDF (bottom right) is a close approximation ($R^2 > 0.99$).}
\label{fig:heatmaps}
\end{figure}

\subsection{Black-box dataset}\label{app:blackbox}

\autoref{fig:r2-black} compares all methods on the 57 black-box problems in terms of median $R^2$, expression complexity, and training time. SMILE achieves the lowest training time among all baselines and produces expressions of very low complexity, second only to DSR. However, SMILE's median $R^2$ is lower than several baselines on the black-box dataset. This is consistent with the design of the pipeline: the data analysis and decomposition stage relies on detecting power-law and compositional structure in the data, assumptions that are well-suited to scientific laws but may not hold for arbitrary real-world relationships where the underlying function lacks such regularity. Despite the lower median $R^2$, the left panel of \autoref{fig:pareto-com} shows that SMILE lies on the Pareto front alongside DSR and PySR, as the expressions it recovers are substantially simpler than those of higher-$R^2$ methods.

\begin{figure}[htb]
\centering
\includegraphics[width=\linewidth]{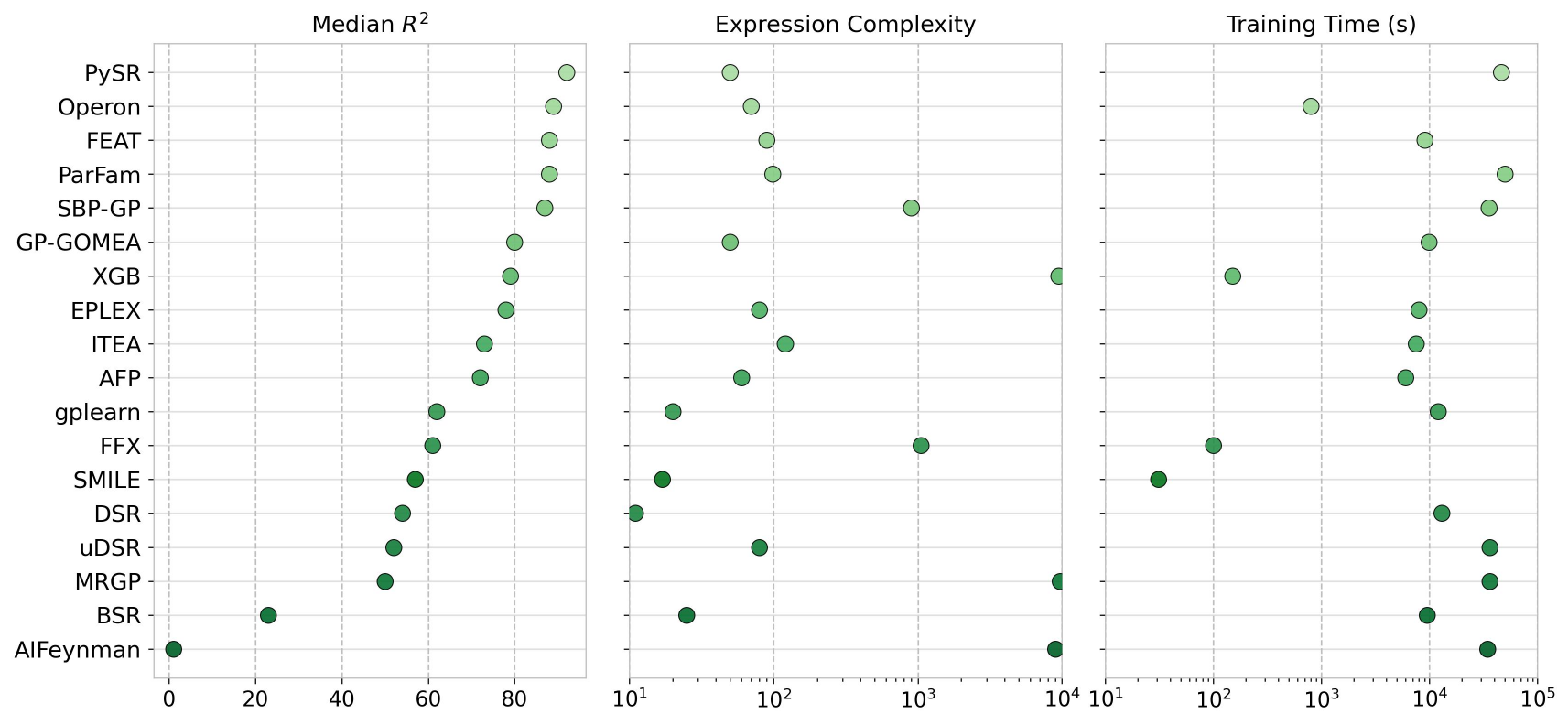}
\caption{Comparison of all methods on the black-box dataset in terms of median $R^2$, expression complexity, and training time.}
\label{fig:r2-black}
\end{figure}

\subsection{Ablation studies}\label{app:ablation}
To evaluate the contribution of each pipeline component, we systematically remove one component at a time while keeping the rest unchanged and re-run the pipeline on the Feynman dataset. The four components evaluated are: variable analysis and decomposition, greedy pruning, parametric optimization, and gradient-based rounding. Since removing a component reduces the simplification capacity of the pipeline, subsequent stages may fail to converge within a reasonable time. We therefore impose a maximum time budget per stage, and if a formula exceeds this limit, it is marked as a timeout error. This situation does not occur in the full pipeline but arises in ablated configurations. \autoref{fig:ablation} reports the SSR, accuracy solution rate, and timeout error rate for each ablated configuration compared to the full pipeline.

The four ablated configurations fall into two groups based on severity. Removing variable analysis or gradient-based rounding reduces SSR by approximately 15\% and the accuracy solution rate by roughly 30\%, yet produces no timeout errors, indicating that the pipeline still terminates within the allotted time budget. Training time roughly doubles in both cases, as downstream stages must compensate for the missing simplification step. Removing variable analysis eliminates the compositional decomposition of the target function, so the network attempts to fit the full expression in a single pass. The resulting formulas tend to be longer and less structured, achieving reasonable regression accuracy but failing to recover the ground-truth symbolic form. Removing gradient-based rounding leaves learned exponents and coefficients as floating-point values, preventing the final expression from reaching the exact symbolic constants required for a successful symbolic recovery.

Removing greedy pruning or parametric optimization has a substantially larger effect. SSR drops by approximately 35\% in both cases, and a significant fraction of formulas incur timeout errors. Without greedy pruning, the full set of gated components passes to parametric optimization and gradient-based rounding, producing expressions with a large number of active coefficients. The combinatorial burden of optimizing and rounding all of these coefficients simultaneously causes downstream stages to exceed the time budget. Without parametric optimization, the coefficient values passed to gradient-based rounding remain far from any sparse symbolic target, making the rounding step unable to collapse terms effectively and again triggering timeouts.

\begin{figure}[htb]
\centering
\includegraphics[width=0.9\linewidth]{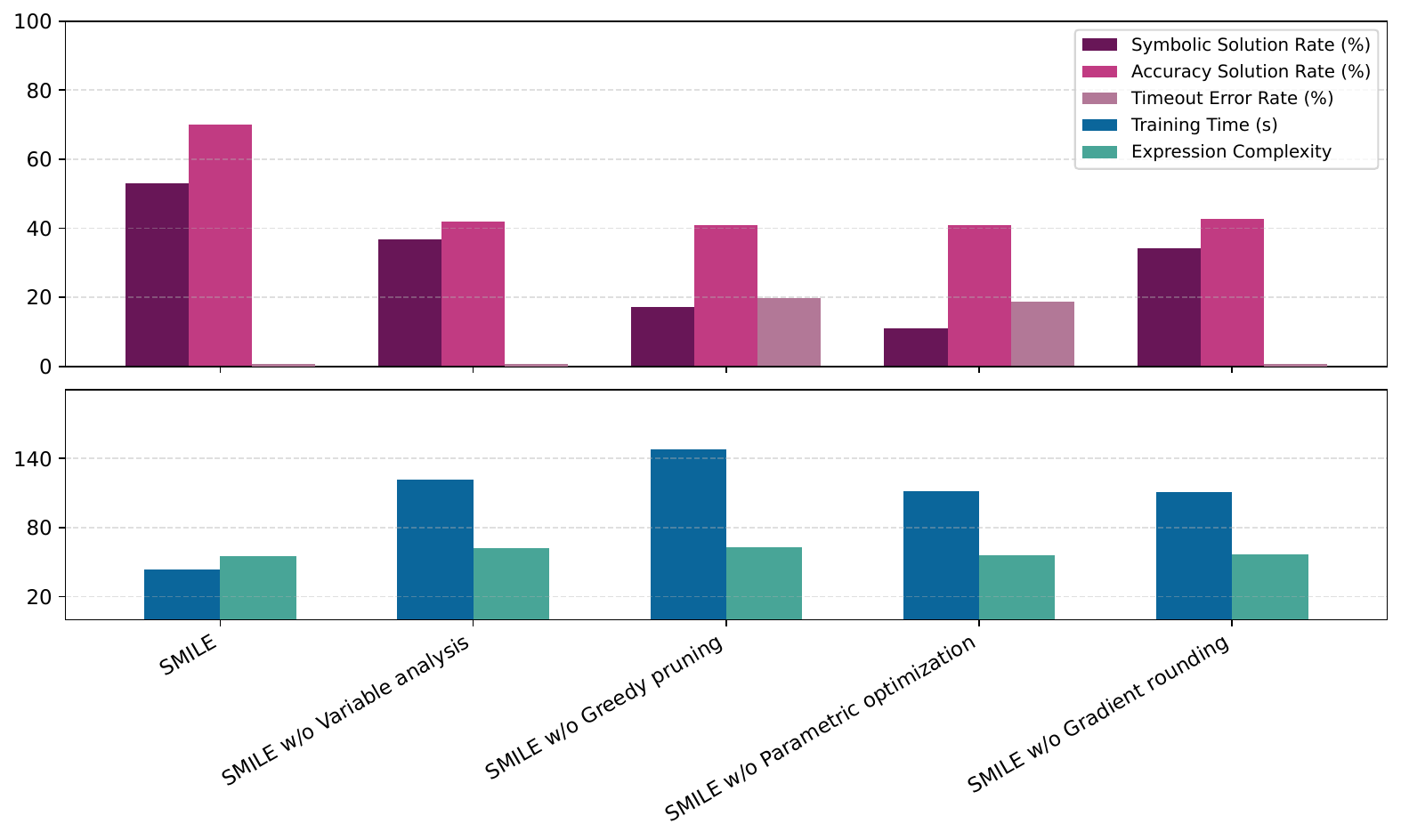}
\caption{Ablation study on the Feynman dataset. Each configuration removes one pipeline component while keeping the rest unchanged. SSR, accuracy solution rate, and timeout error rate are reported.}
\label{fig:ablation}
\end{figure}

Expression complexity increases across all four configurations. When variable analysis is removed, the loss of compositional structure yields long formulas that fit the data but do not reflect the underlying equation. When greedy pruning is removed, redundant components persist through the remainder of the pipeline, inflating the final expression. When parametric optimization or gradient-based rounding is removed, the pipeline loses a simplification stage that would otherwise drive coefficients toward exact values, leaving additional unresolved terms in the output.

Taken together, these results confirm that each component contributes to both the accuracy and the tractability of the discovery pipeline. Greedy pruning and parametric optimization are the most critical for maintaining a manageable search space, while variable analysis and gradient-based rounding primarily improve the quality and interpretability of the recovered expressions.


\newpage
\section*{NeurIPS Paper Checklist}

\begin{enumerate}

\item {\bf Claims}
    \item[] Question: Do the main claims made in the abstract and introduction accurately reflect the paper's contributions and scope?
    \item[] Answer: \answerYes{} 
    \item[] Justification: The abstract and introduction clearly state the main contributions of the paper, including the SMILE framework, the data-driven decomposition pipeline, and the evaluation on SRBench. These claims are consistent with the technical content of the paper: the method is described in detail in Section 3, and the experimental results in Section 4 and the appendix provide direct support. The paper does not make claims beyond what is demonstrated in the experiments.
    \item[] Guidelines:
    \begin{itemize}
        \item The answer \answerNA{} means that the abstract and introduction do not include the claims made in the paper.
        \item The abstract and/or introduction should clearly state the claims made, including the contributions made in the paper and important assumptions and limitations. A \answerNo{} or \answerNA{} answer to this question will not be perceived well by the reviewers. 
        \item The claims made should match theoretical and experimental results, and reflect how much the results can be expected to generalize to other settings. 
        \item It is fine to include aspirational goals as motivation as long as it is clear that these goals are not attained by the paper. 
    \end{itemize}

\item {\bf Limitations}
    \item[] Question: Does the paper discuss the limitations of the work performed by the authors?
    \item[] Answer: \answerYes{} 
    \item[] Justification: Limitations are discussed in Section 5, including the lower accuracy solution rate due to the shallow architecture and the pipeline's reliance on compositional structure assumptions, which may not hold for arbitrary black-box datasets.
    \item[] Guidelines:
    \begin{itemize}
        \item The answer \answerNA{} means that the paper has no limitation while the answer \answerNo{} means that the paper has limitations, but those are not discussed in the paper. 
        \item The authors are encouraged to create a separate ``Limitations'' section in their paper.
        \item The paper should point out any strong assumptions and how robust the results are to violations of these assumptions (e.g., independence assumptions, noiseless settings, model well-specification, asymptotic approximations only holding locally). The authors should reflect on how these assumptions might be violated in practice and what the implications would be.
        \item The authors should reflect on the scope of the claims made, e.g., if the approach was only tested on a few datasets or with a few runs. In general, empirical results often depend on implicit assumptions, which should be articulated.
        \item The authors should reflect on the factors that influence the performance of the approach. For example, a facial recognition algorithm may perform poorly when image resolution is low or images are taken in low lighting. Or a speech-to-text system might not be used reliably to provide closed captions for online lectures because it fails to handle technical jargon.
        \item The authors should discuss the computational efficiency of the proposed algorithms and how they scale with dataset size.
        \item If applicable, the authors should discuss possible limitations of their approach to address problems of privacy and fairness.
        \item While the authors might fear that complete honesty about limitations might be used by reviewers as grounds for rejection, a worse outcome might be that reviewers discover limitations that aren't acknowledged in the paper. The authors should use their best judgment and recognize that individual actions in favor of transparency play an important role in developing norms that preserve the integrity of the community. Reviewers will be specifically instructed to not penalize honesty concerning limitations.
    \end{itemize}

\item {\bf Theory assumptions and proofs}
    \item[] Question: For each theoretical result, does the paper provide the full set of assumptions and a complete (and correct) proof?
    \item[] Answer: \answerYes{} 
    \item[] Justification: The universal approximation theorem is stated in Section 3 with a proof sketch, and the full proof is provided in Appendix A. The gradient-based rounding theorem is stated in Section 3 with the complete proof in Appendix B.
    \item[] Guidelines:
    \begin{itemize}
        \item The answer \answerNA{} means that the paper does not include theoretical results. 
        \item All the theorems, formulas, and proofs in the paper should be numbered and cross-referenced.
        \item All assumptions should be clearly stated or referenced in the statement of any theorems.
        \item The proofs can either appear in the main paper or the supplemental material, but if they appear in the supplemental material, the authors are encouraged to provide a short proof sketch to provide intuition. 
        \item Inversely, any informal proof provided in the core of the paper should be complemented by formal proofs provided in appendix or supplemental material.
        \item Theorems and Lemmas that the proof relies upon should be properly referenced. 
    \end{itemize}

    \item {\bf Experimental result reproducibility}
    \item[] Question: Does the paper fully disclose all the information needed to reproduce the main experimental results of the paper to the extent that it affects the main claims and/or conclusions of the paper (regardless of whether the code and data are provided or not)?
    \item[] Answer: \answerYes{} 
    \item[] Justification: The paper describes the full methodology including the network architecture, training procedure, loss function, and discovery pipeline in Section 3. The complete algorithm is provided in Appendix B. All datasets are publicly available through SRBench (\url{https://github.com/cavalab/srbench}) and PMLB (\url{https://github.com/EpistasisLab/pmlb}), and hyperparameter details are provided in Appendix E.
    \item[] Guidelines:
    \begin{itemize}
        \item The answer \answerNA{} means that the paper does not include experiments.
        \item If the paper includes experiments, a \answerNo{} answer to this question will not be perceived well by the reviewers: Making the paper reproducible is important, regardless of whether the code and data are provided or not.
        \item If the contribution is a dataset and\slash or model, the authors should describe the steps taken to make their results reproducible or verifiable. 
        \item Depending on the contribution, reproducibility can be accomplished in various ways. For example, if the contribution is a novel architecture, describing the architecture fully might suffice, or if the contribution is a specific model and empirical evaluation, it may be necessary to either make it possible for others to replicate the model with the same dataset, or provide access to the model. In general. releasing code and data is often one good way to accomplish this, but reproducibility can also be provided via detailed instructions for how to replicate the results, access to a hosted model (e.g., in the case of a large language model), releasing of a model checkpoint, or other means that are appropriate to the research performed.
        \item While NeurIPS does not require releasing code, the conference does require all submissions to provide some reasonable avenue for reproducibility, which may depend on the nature of the contribution. For example
        \begin{enumerate}
            \item If the contribution is primarily a new algorithm, the paper should make it clear how to reproduce that algorithm.
            \item If the contribution is primarily a new model architecture, the paper should describe the architecture clearly and fully.
            \item If the contribution is a new model (e.g., a large language model), then there should either be a way to access this model for reproducing the results or a way to reproduce the model (e.g., with an open-source dataset or instructions for how to construct the dataset).
            \item We recognize that reproducibility may be tricky in some cases, in which case authors are welcome to describe the particular way they provide for reproducibility. In the case of closed-source models, it may be that access to the model is limited in some way (e.g., to registered users), but it should be possible for other researchers to have some path to reproducing or verifying the results.
        \end{enumerate}
    \end{itemize}

\item {\bf Open access to data and code}
    \item[] Question: Does the paper provide open access to the data and code, with sufficient instructions to faithfully reproduce the main experimental results, as described in supplemental material?
    \item[] Answer: \answerYes{} 
    \item[] Justification: The code will be made publicly available upon acceptance. All datasets used are publicly available through SRBench (\url{https://github.com/cavalab/srbench}) and PMLB (\url{https://github.com/EpistasisLab/pmlb}).
    \item[] Guidelines:
    \begin{itemize}
        \item The answer \answerNA{} means that paper does not include experiments requiring code.
        \item Please see the NeurIPS code and data submission guidelines (\url{https://neurips.cc/public/guides/CodeSubmissionPolicy}) for more details.
        \item While we encourage the release of code and data, we understand that this might not be possible, so \answerNo{} is an acceptable answer. Papers cannot be rejected simply for not including code, unless this is central to the contribution (e.g., for a new open-source benchmark).
        \item The instructions should contain the exact command and environment needed to run to reproduce the results. See the NeurIPS code and data submission guidelines (\url{https://neurips.cc/public/guides/CodeSubmissionPolicy}) for more details.
        \item The authors should provide instructions on data access and preparation, including how to access the raw data, preprocessed data, intermediate data, and generated data, etc.
        \item The authors should provide scripts to reproduce all experimental results for the new proposed method and baselines. If only a subset of experiments are reproducible, they should state which ones are omitted from the script and why.
        \item At submission time, to preserve anonymity, the authors should release anonymized versions (if applicable).
        \item Providing as much information as possible in supplemental material (appended to the paper) is recommended, but including URLs to data and code is permitted.
    \end{itemize}

\item {\bf Experimental setting/details}
    \item[] Question: Does the paper specify all the training and test details (e.g., data splits, hyperparameters, how they were chosen, type of optimizer) necessary to understand the results?
    \item[] Answer: \answerYes{} 
    \item[] Justification: Training details including the loss function, gating mechanism, and optimization procedure are described in Section 3. Noise adding procedure, evaluation metrics, and dataset restrictions are specified in Section 4, with additional details in Appendix E.
    \item[] Guidelines:
    \begin{itemize}
        \item The answer \answerNA{} means that the paper does not include experiments.
        \item The experimental setting should be presented in the core of the paper to a level of detail that is necessary to appreciate the results and make sense of them.
        \item The full details can be provided either with the code, in appendix, or as supplemental material.
    \end{itemize}

\item {\bf Experiment statistical significance}
    \item[] Question: Does the paper report error bars suitably and correctly defined or other appropriate information about the statistical significance of the experiments?
    \item[] Answer: \answerYes{} 
    \item[] Justification: All results on ground-truth datasets are averaged over three independent trials as stated in Section 4. For black-box datasets, we report median $R^2$ and median complexity following the SRBench evaluation protocol.
    \item[] Guidelines:
    \begin{itemize}
        \item The answer \answerNA{} means that the paper does not include experiments.
        \item The authors should answer \answerYes{} if the results are accompanied by error bars, confidence intervals, or statistical significance tests, at least for the experiments that support the main claims of the paper.
        \item The factors of variability that the error bars are capturing should be clearly stated (for example, train/test split, initialization, random drawing of some parameter, or overall run with given experimental conditions).
        \item The method for calculating the error bars should be explained (closed form formula, call to a library function, bootstrap, etc.)
        \item The assumptions made should be given (e.g., Normally distributed errors).
        \item It should be clear whether the error bar is the standard deviation or the standard error of the mean.
        \item It is OK to report 1-sigma error bars, but one should state it. The authors should preferably report a 2-sigma error bar than state that they have a 96\% CI, if the hypothesis of Normality of errors is not verified.
        \item For asymmetric distributions, the authors should be careful not to show in tables or figures symmetric error bars that would yield results that are out of range (e.g., negative error rates).
        \item If error bars are reported in tables or plots, the authors should explain in the text how they were calculated and reference the corresponding figures or tables in the text.
    \end{itemize}

\item {\bf Experiments compute resources}
    \item[] Question: For each experiment, does the paper provide sufficient information on the computer resources (type of compute workers, memory, time of execution) needed to reproduce the experiments?
    \item[] Answer: \answerYes{} 
    \item[] Justification: Training time is reported for all experiments in Section 4 and Appendix E. Compute resources and hardware details are provided in Appendix E.
    \item[] Guidelines:
    \begin{itemize}
        \item The answer \answerNA{} means that the paper does not include experiments.
        \item The paper should indicate the type of compute workers CPU or GPU, internal cluster, or cloud provider, including relevant memory and storage.
        \item The paper should provide the amount of compute required for each of the individual experimental runs as well as estimate the total compute. 
        \item The paper should disclose whether the full research project required more compute than the experiments reported in the paper (e.g., preliminary or failed experiments that didn't make it into the paper). 
    \end{itemize}
    
\item {\bf Code of ethics}
    \item[] Question: Does the research conducted in the paper conform, in every respect, with the NeurIPS Code of Ethics \url{https://neurips.cc/public/EthicsGuidelines}?
    \item[] Answer: \answerYes{}
    \item[] Justification: The research conforms with the NeurIPS Code of Ethics and all the datasets used are publicly available benchmarks.
    \item[] Guidelines:
    \begin{itemize}
        \item The answer \answerNA{} means that the authors have not reviewed the NeurIPS Code of Ethics.
        \item If the authors answer \answerNo, they should explain the special circumstances that require a deviation from the Code of Ethics.
        \item The authors should make sure to preserve anonymity (e.g., if there is a special consideration due to laws or regulations in their jurisdiction).
    \end{itemize}

\item {\bf Broader impacts}
    \item[] Question: Does the paper discuss both potential positive societal impacts and negative societal impacts of the work performed?
    \item[] Answer: \answerNA{}
    \item[] Justification: This work proposes a method for symbolic regression, a fundamental scientific discovery tool. It poses no direct negative societal impact beyond those common to general machine learning research.
    \item[] Guidelines:
    \begin{itemize}
        \item The answer \answerNA{} means that there is no societal impact of the work performed.
        \item If the authors answer \answerNA{} or \answerNo, they should explain why their work has no societal impact or why the paper does not address societal impact.
        \item Examples of negative societal impacts include potential malicious or unintended uses (e.g., disinformation, generating fake profiles, surveillance), fairness considerations (e.g., deployment of technologies that could make decisions that unfairly impact specific groups), privacy considerations, and security considerations.
        \item The conference expects that many papers will be foundational research and not tied to particular applications, let alone deployments. However, if there is a direct path to any negative applications, the authors should point it out. For example, it is legitimate to point out that an improvement in the quality of generative models could be used to generate Deepfakes for disinformation. On the other hand, it is not needed to point out that a generic algorithm for optimizing neural networks could enable people to train models that generate Deepfakes faster.
        \item The authors should consider possible harms that could arise when the technology is being used as intended and functioning correctly, harms that could arise when the technology is being used as intended but gives incorrect results, and harms following from (intentional or unintentional) misuse of the technology.
        \item If there are negative societal impacts, the authors could also discuss possible mitigation strategies (e.g., gated release of models, providing defenses in addition to attacks, mechanisms for monitoring misuse, mechanisms to monitor how a system learns from feedback over time, improving the efficiency and accessibility of ML).
    \end{itemize}
    
\item {\bf Safeguards}
    \item[] Question: Does the paper describe safeguards that have been put in place for responsible release of data or models that have a high risk for misuse (e.g., pre-trained language models, image generators, or scraped datasets)?
    \item[] Answer: \answerNA{}
    \item[] Justification: The proposed method is a symbolic regression framework that does not pose risks of misuse. It does not involve pre-trained language models, image generators, or scraped datasets.
    \item[] Guidelines:
    \begin{itemize}
        \item The answer \answerNA{} means that the paper poses no such risks.
        \item Released models that have a high risk for misuse or dual-use should be released with necessary safeguards to allow for controlled use of the model, for example by requiring that users adhere to usage guidelines or restrictions to access the model or implementing safety filters. 
        \item Datasets that have been scraped from the Internet could pose safety risks. The authors should describe how they avoided releasing unsafe images.
        \item We recognize that providing effective safeguards is challenging, and many papers do not require this, but we encourage authors to take this into account and make a best faith effort.
    \end{itemize}

\item {\bf Licenses for existing assets}
    \item[] Question: Are the creators or original owners of assets (e.g., code, data, models), used in the paper, properly credited and are the license and terms of use explicitly mentioned and properly respected?
    \item[] Answer: \answerYes{}
    \item[] Justification: All datasets and benchmarks used are properly cited. SRBench and PMLB are publicly available open-source resources, and all baseline methods are credited with their original publications.
    \item[] Guidelines:
    \begin{itemize}
        \item The answer \answerNA{} means that the paper does not use existing assets.
        \item The authors should cite the original paper that produced the code package or dataset.
        \item The authors should state which version of the asset is used and, if possible, include a URL.
        \item The name of the license (e.g., CC-BY 4.0) should be included for each asset.
        \item For scraped data from a particular source (e.g., website), the copyright and terms of service of that source should be provided.
        \item If assets are released, the license, copyright information, and terms of use in the package should be provided. For popular datasets, \url{paperswithcode.com/datasets} has curated licenses for some datasets. Their licensing guide can help determine the license of a dataset.
        \item For existing datasets that are re-packaged, both the original license and the license of the derived asset (if it has changed) should be provided.
        \item If this information is not available online, the authors are encouraged to reach out to the asset's creators.
    \end{itemize}

\item {\bf New assets}
    \item[] Question: Are new assets introduced in the paper well documented and is the documentation provided alongside the assets?
    \item[] Answer: \answerNA{}
    \item[] Justification: The paper does not introduce new datasets or pre-trained models. The code will be released upon acceptance with accompanying documentation.
    \item[] Guidelines:
    \begin{itemize}
        \item The answer \answerNA{} means that the paper does not release new assets.
        \item Researchers should communicate the details of the dataset\slash code\slash model as part of their submissions via structured templates. This includes details about training, license, limitations, etc. 
        \item The paper should discuss whether and how consent was obtained from people whose asset is used.
        \item At submission time, remember to anonymize your assets (if applicable). You can either create an anonymized URL or include an anonymized zip file.
    \end{itemize}

\item {\bf Crowdsourcing and research with human subjects}
    \item[] Question: For crowdsourcing experiments and research with human subjects, does the paper include the full text of instructions given to participants and screenshots, if applicable, as well as details about compensation (if any)? 
    \item[] Answer: \answerNA{}
    \item[] Justification: The paper does not involve crowdsourcing or research with human subjects.
    \item[] Guidelines:
    \begin{itemize}
        \item The answer \answerNA{} means that the paper does not involve crowdsourcing nor research with human subjects.
        \item Including this information in the supplemental material is fine, but if the main contribution of the paper involves human subjects, then as much detail as possible should be included in the main paper. 
        \item According to the NeurIPS Code of Ethics, workers involved in data collection, curation, or other labor should be paid at least the minimum wage in the country of the data collector. 
    \end{itemize}

\item {\bf Institutional review board (IRB) approvals or equivalent for research with human subjects}
    \item[] Question: Does the paper describe potential risks incurred by study participants, whether such risks were disclosed to the subjects, and whether Institutional Review Board (IRB) approvals (or an equivalent approval/review based on the requirements of your country or institution) were obtained?
    \item[] Answer: \answerNA{}
    \item[] Justification: The paper does not involve research with human subjects.
    \item[] Guidelines:
    \begin{itemize}
        \item The answer \answerNA{} means that the paper does not involve crowdsourcing nor research with human subjects.
        \item Depending on the country in which research is conducted, IRB approval (or equivalent) may be required for any human subjects research. If you obtained IRB approval, you should clearly state this in the paper. 
        \item We recognize that the procedures for this may vary significantly between institutions and locations, and we expect authors to adhere to the NeurIPS Code of Ethics and the guidelines for their institution. 
        \item For initial submissions, do not include any information that would break anonymity (if applicable), such as the institution conducting the review.
    \end{itemize}

\item {\bf Declaration of LLM usage}
    \item[] Question: Does the paper describe the usage of LLMs if it is an important, original, or non-standard component of the core methods in this research? Note that if the LLM is used only for writing, editing, or formatting purposes and does \emph{not} impact the core methodology, scientific rigor, or originality of the research, declaration is not required.
    \item[] Answer: \answerNA{}
    \item[] Justification: LLMs were used only for editing and polishing the manuscript text. They are not a component of the core methodology.
    \item[] Guidelines:
    \begin{itemize}
        \item The answer \answerNA{} means that the core method development in this research does not involve LLMs as any important, original, or non-standard components.
        \item Please refer to our LLM policy in the NeurIPS handbook for what should or should not be described.
    \end{itemize}

\end{enumerate}

\end{document}